\documentclass{article}
\pdfoutput=1

     \PassOptionsToPackage{numbers, compress}{natbib}

     \usepackage[preprint]{neurips_2021}

\usepackage[utf8]{inputenc} 
\usepackage[T1]{fontenc}    
\usepackage{url}            
\usepackage{booktabs}       
\usepackage{amsfonts}       
\usepackage{nicefrac}       
\usepackage{microtype}      
\usepackage{xcolor}         
\usepackage{hyperref}
\usepackage{graphicx}
 \graphicspath{{figures/}} 
\usepackage{float}
\newfloat{figtab}{htb}{fgtb}
\makeatletter
  \newcommand\figcaption{\def\@captype{figure}\caption}
  \newcommand\tabcaption{\def\@captype{table}\caption}
\makeatother
\usepackage{bbding}
\usepackage{amsmath}
\usepackage{amssymb}
\usepackage{amsthm}
\usepackage{subfigure}
\usepackage{wrapfig}        
\usepackage{algorithm}
\usepackage{algorithmic}
\newtheorem{theorem}{Theorem}
\newtheorem{lemma}{Lemma}

\newtheorem{definition}{Definition}

\title{Efficient Continuous Control with Double Actors and Regularized Critics}

\author{%
  Jiafei Lyu$^{1}$, Xiaoteng Ma$^{2}$, Jiangpeng Yan$^{2}$, Xiu Li$^{1}$ \\
  $^{1}$Tsinghua Shenzhen International Graduate School, Tsinghua University\\
  $^{2}$Department of Automation, Tsinghua Unversity \\
  \texttt{\{lvjf20,ma-xt17,yanjp17\}@mails.tsinghua.edu.cn, li.xiu@sz.tsinghua.edu.cn} \\
}

\begin{document}

\maketitle

\begin{abstract}
  How to obtain good value estimation is one of the key problems in Reinforcement Learning (RL). Current value estimation methods, such as DDPG and TD3, suffer from unnecessary over- or underestimation bias. In this paper, we explore the potential of double actors, which has been neglected for a long time, for better value function estimation in continuous setting. First, we uncover and demonstrate the bias alleviation property of double actors by building double actors upon single critic and double critics to handle overestimation bias in DDPG and underestimation bias in TD3 respectively. Next, we interestingly find that double actors help improve the exploration ability of the agent. Finally, to mitigate the uncertainty of value estimate from double critics, we further propose to regularize the critic networks under double actors architecture, which gives rise to Double Actors Regularized Critics (DARC) algorithm. Extensive experimental results on challenging continuous control tasks show that DARC significantly outperforms state-of-the-art methods with higher sample efficiency.
\end{abstract}

\section{Introduction}

Actor-Critic methods \cite{konda2000actor, prokhorov1997adaptive} are among the most popular methods in Reinforcement Learning (RL) \cite{sutton2018reinforcement} which involves value function estimation \cite{baird1995residual, gordon1995stable} and policy gradients \cite{weng2018PG, williams1992simple}. Deep Deterministic Policy Gradient (DDPG) \cite{lillicrap2015continuous} is a typical and widely-used RL algorithm for continuous control that is built upon actor-critic method. It has been revealed that DDPG results in severe \emph{overestimation} bias \cite{fujimoto2018addressing} using single critic for function approximation \cite{thrun1993issues} as the actor network is trained to produce the action with the highest value estimate. Such phenomenon is widely reported in the methods based on Q-learning \cite{hasselt2010double, pandey2010approximate, watkins1989learning}, which does harm to their performance. Fujimoto et al. \cite{fujimoto2018addressing} borrow ideas from Double Q-learning algorithm \cite{hasselt2010double, hasselt2016deep} and propose Twin Delayed Deep Deterministic Policy Gradient (TD3) to address the issue in DDPG. TD3 utilizes the minimum value estimate from double critic networks for value estimation. With the aid of target networks, delayed policy update as well as clipped double Q-learning, TD3 alleviates the overestimation bias problem and significantly improves the performance of DDPG. However, it turns out that TD3 may lead to large \emph{underestimation} bias \cite{ciosek2019better, pan2020softmax}, which negatively affects its performance.

In this paper, we explore the advantages of double actors for value estimation correction and how they benefit continuous control, which has been neglected for a long time. We show that double actors help relieve overestimation bias in DDPG if built upon single critic, and underestimation bias in TD3 if built upon double critics, where we develop Double Actors DDPG (DADDPG) algorithm and Double Actors TD3 (DATD3) algorithm respectively. We experimentally find out that DADDPG significantly outperforms DDPG, which sheds light on the potential advantages of double actors. 

Besides the benefit of double actors in value estimation, we also uncover an important property of double actors, i.e., they enhance the exploration ability of the agent. Double actors offer double paths for policy optimization instead of making the agent confined by a single policy, which significantly reduces the probability of the actor being trapped locally. Furthermore, the second critic is kind of wasted in TD3 as it is only used for value correction and the actor relies on the first critic for updating and action execution. Double actors architecture can fully utilize critics, where each critic is in charge of one actor (see details in Appendix \ref{fig:graphdifference}). 

To alleviate the uncertainty of value estimate from double critics, we propose to regularize the critics, where a convex combination of value estimates from double actors is used for better flexibility, leading to the Double Actors Regularized Critics (DARC) algorithm. We give an upper bound of the error between the optimal value and the value function with double actors, showing that the error can be efficiently controlled with critic regularization. For better illustration on benefits of applying double actors, we compare components of different algorithms and report the average performance improvement on MuJoCo \cite{todorov2012mujoco} environments (see Table \ref{tab:structurecomparison} for details).

We perform extensive experiments on two challenging continuous control benchmarks, OpenAI Gym \cite{brockman2016openai} and PyBullet Gym \cite{benelot2018}, where we compare our DARC algorithm against the current state-of-the-art methods including TD3 and Soft Actor-Critic (SAC) \cite{haarnoja2018soft, haarnoja2018softactorcritic}. The results show that DARC significantly outperforms them with much higher sample efficiency.

\begin{table}
    \tabcaption{Algorithmic component comparison and average performance improvement compared to DDPG baseline. The improvement refers to the averaged relative improvement on the mean final scores of MuJoCo environments with respect to DDPG baseline over 5 runs.}
    \label{tab:structurecomparison}
    \small
    \centering
    \setlength\tabcolsep{3pt}
    \begin{tabular}{ccccccc}
      \toprule
    Algorithms  & Double Actors & Double Critics  & Value Correction & Regularization & Improvement \\
    \midrule
    DDPG\cite{lillicrap2015continuous}  & \XSolidBrush & \XSolidBrush & \XSolidBrush & \XSolidBrush & 100\% \\
    TD3\cite{fujimoto2018addressing}   & \XSolidBrush & \CheckmarkBold & \CheckmarkBold & \XSolidBrush & 245\% \\
    SAC\cite{haarnoja2018soft}   & \XSolidBrush & \CheckmarkBold & \CheckmarkBold & \CheckmarkBold &  191\% \\
    DADDPG (this work) & \CheckmarkBold & \XSolidBrush  & \CheckmarkBold & \XSolidBrush & 167\% \\
    DATD3 (this work) & \CheckmarkBold & \CheckmarkBold & \CheckmarkBold & \XSolidBrush & 291\% \\
    DARC (this work) & \CheckmarkBold & \CheckmarkBold & \CheckmarkBold & \CheckmarkBold & \textbf{331\%} \\
\bottomrule
        \end{tabular}
\end{table}

\section{Preliminaries}
\label{sec:pre}

Reinforcement learning studies sequential decision making problems and it can be formulated by a Markov Decision Process (MDP), which can be defined as a 5-tuple $\langle \mathcal{S},\mathcal{A},p,r,\gamma\rangle$ where $\mathcal{S},\mathcal{A}$ denote state space and action space respectively, $p$ denotes transition probability, $r:\mathcal{S}\times\mathcal{A}\mapsto\mathbb{R}$ is the reward function and $\gamma \in[0,1)$ is the discount factor. The agent behaves according to the policy $\pi_\phi:\mathcal{S}\mapsto\mathcal{A}$ parameterized by $\phi$. The objective function of reinforcement learning can be written as $J(\phi) = \mathbb{E}_s[\sum_{t=0}^\infty \gamma^t r(s_t,a_t)|s_0,a_0;\pi_{\phi}(s)]$, which aims at maximizing expected future discounted rewards following policy $\pi_{\phi}(s)$.

We consider continuous control scenario with bounded action space and we further assume a continuous and bounded reward function $r$. Then the policy $\pi_\phi$ can be improved by conducting policy gradient ascending in terms of the objective function $J(\phi)$. The Deterministic Policy Gradient (DPG) Theorem \cite{silver2014deterministic} offers a practical way of calculating the gradient:
\begin{equation}
\label{eq:dpg}
    \nabla_\phi J(\phi) = \mathbb{E}_s [\nabla_\phi \pi_{\phi}(s) \nabla_a Q_{\theta}(s,a)|_{a=\pi_{\phi}(s)}],
\end{equation}
where $Q_{\theta}(s,a)$ is the $Q$-function with parameter $\theta$ that approximates the long-term rewards given state and action. In actor-critic architecture, the critic estimates value function $Q_{\theta}(s,a)$ to approximate the true parameter $\theta^{\mathrm{true}}$, and the actor is updated using Eq. (\ref{eq:dpg}). DDPG learns a deterministic policy $\pi_{\phi}(s)$ to approximate the optimal policy as it is expensive to directly apply the max operator over the continuous action space $\mathcal{A}$. With TD-learning, the critic in DDPG is updated via $\tilde{\theta} \leftarrow \theta + \eta \mathbb{E}_{s,a\sim \rho}[r+\gamma Q_{\theta^\prime}(s^\prime,\pi_{\phi^\prime}(s^\prime)) - Q_{\theta}(s,a)]\nabla_\theta Q_{\theta}(s,a)$, where $\eta$ is the learning rate, $\rho$ is the sample distribution in the replay buffer, $\phi^\prime$ and $\theta^\prime$ are the parameters of the target actor and critic network respectively. TD3 addresses the overestimation problem in DDPG by employing double critics for value estimation, which is given by $\hat{Q}(s^\prime,a^\prime) \leftarrow \min_{i=1,2} Q_{\theta^\prime_i}(s^\prime,a^\prime)$, and one actor for policy improvement. Denote $\mathcal{T}(s^\prime)$ as the value estimation function that is utilized to estimate the target value $r+\gamma\mathcal{T}(s^\prime)$, and then we have $\mathcal{T}_{\mathrm{DDPG}}(s^\prime)=Q_{\theta^\prime}(s^\prime,\pi_{\phi^\prime}(s^\prime))$ and $\mathcal{T}_{\mathrm{TD3}}(s^\prime)=\min_{i=1,2} Q_{\theta^\prime_i}(s^\prime,\pi_{\phi^\prime}(s^\prime))$. 

\section{Double Actors for Better Continuous Control}

In this section, we discuss how to better estimate value function with double actors and also demonstrate that double actors can help ease overestimation bias in DDPG as well as the underestimation bias in TD3. We further show that double actors could induce the stronger exploration capability in continuous control settings.

\subsection{Better Value Estimation with Double Actors}
\label{sec:doubleactorvaluecorrection}
It is vital for any actor-critic-style algorithm to estimate the value function in a good manner such that the actor can be driven to learn a better policy. However, existing methods are far from satisfying.

\noindent\textbf{How to use double actors for value estimation correction?} If we build double actors upon single Q-network (the critic) parameterized by $\theta$, then for each training step, the critic has two paths to choose from: either following $\pi_{\phi_1^\prime}(s)$ or $\pi_{\phi_2^\prime}(s)$, where both paths are estimated to approximate the optimal path for the task. We then propose to estimate the value function via
\begin{equation}
\label{eq:daddpg}
    \hat{V}(s^\prime) \leftarrow \min_{i=1,2} Q_{\theta^\prime}(s^\prime, \pi_{\phi_i^\prime}(s^\prime)),
\end{equation}
where $\theta^\prime,\phi_i^\prime,i\in\{1,2\}$ are the parameters of the target networks, which leads to \emph{Double Actors DDPG} (DADDPG) algorithm (see Appendix \ref{sec:doubleactoronecritic} for detailed algorithm and its comparison to TD3).  While if we build double actors upon double critics networks parameterized by $\theta_1^\prime$, $\theta_2^\prime$ respectively, one naive way to estimate the value function would be taking minimum of Q-networks for each policy $\pi_{\phi_i^\prime}(s), i=1,2$, and employing the maximal one for final value estimation, i.e.,
\begin{equation}
\label{eq:datd3}
    \hat{V}(s^\prime) \leftarrow \max \left\{ \min_{i=1,2}\left(Q_{\theta_i^\prime}(s^\prime,\pi_{\phi_1^\prime}(s^\prime))\right), \min_{j=1,2}\left(Q_{\theta_j^\prime}(s^\prime,\pi_{\phi_2^\prime}(s^\prime))\right) \right\},
\end{equation}
which we refer as \emph{Double Actors TD3} (DATD3) algorithm (see Appendix \ref{sec:doubleactorsdoublecritics} for more details). Eq. (\ref{eq:daddpg}) and Eq. (\ref{eq:datd3}) provide us with a novel way of estimating value function upon single and double critics. One may wonder why we adopt value estimation using Eq. (\ref{eq:datd3}) instead of taking maximal value for each critic under the two policies and employing smaller estimation for target value update, i.e., $\hat{V}(s^\prime) \leftarrow \min \left\{ \max_{i=1,2}\left(Q_{\theta_1}(s^\prime,\pi_{\phi_i^\prime}(s^\prime))\right), \max_{j=1,2}\left(Q_{\theta_2}(s^\prime,\pi_{\phi_j^\prime}(s^\prime))\right) \right\}$. However, if we use double actors for value estimation correction firstly, the actors would have no capability of correcting the value estimation in a good manner as the critic itself may fail to predict the return accurately. Actors are guided and driven by its counterpart critic and the actor can only be as good as allowed by the critic. Moreover, such update scheme may induce large overestimation bias by taking maximum as DDPG. It is worth noting that our method is different from the double Q-learning algorithm as we adopt double actors for value estimation correction instead of constructing double target values for individual update of actor-critic pair. 

\noindent\textbf{What benefits can value estimation with double actors bring?} We demonstrate that double actors help mitigate the severe overestimation problem in DDPG and the underestimation bias in TD3. By the definition in Section \ref{sec:pre}, we have $\mathcal{T}_{\mathrm{DADDPG}}(s^\prime) = \min_{i=1,2} Q_{\theta^\prime}(s^\prime, \pi_{\phi_i^\prime}(s^\prime)), \mathcal{T}_{\mathrm{DATD3}}(s^\prime) = \max \left\{ Q_1(s^\prime,\pi_{\phi_1^\prime}(s^\prime)), Q_2(s^\prime,\pi_{\phi_2^\prime}(s^\prime)) \right\}$, where $Q_1(s^\prime,\pi_{\phi_1^\prime}(s^\prime)) = \min_{i=1,2}\left\{Q_{\theta_i^\prime}(s^\prime,\pi_{\phi_1^\prime}(s^\prime))\right\}$, $Q_2(s^\prime,\pi_{\phi_2^\prime}(s^\prime)) = \min_{k=1,2}\left\{Q_{\theta_k^\prime}(s^\prime,\pi_{\phi_2^\prime}(s^\prime))\right\}$ respectively. Then for DADDPG (one-critic-double-actor structure), we show in Theorem \ref{theo:daddpg} that double actors effectively alleviate the overestimation bias problem in DDPG, and the proof is deferred to Appendix \ref{sec:theorem1}.

\begin{theorem}
\label{theo:daddpg}
Denote the value estimation bias deviating the true value induced by $\mathcal{T}$ as bias($\mathcal{T}$) = $\mathbb{E}[\mathcal{T}(s^\prime)] - \mathbb{E}[Q_{\theta^{\rm{true}}}(s^\prime,\pi_{\phi^\prime}(s^\prime))]$, then we have \rm{bias}($\mathcal{T}_{\mathrm{DADDPG}}$) $\le$ \rm{bias}($\mathcal{T}_{\mathrm{DDPG}}$).
\end{theorem}

\textbf{Remark}: This theorem sheds light to the value estimation correction with double actors as it holds without any special requirement or assumption on double actors, which indicates that DADDPG naturally eases the overestimation bias in DDPG. TD3 also alleviates the overestimation issues while it leverages double critic networks for value correction, which differs from that of DADDPG.

Similarly, we present the relationship of the value estimation of DATD3 (double-actor-double-critic structure) and TD3 in Theorem \ref{theo:datd3}, where the proof is available in Appendix \ref{sec:theorem2}.

\begin{theorem}
\label{theo:datd3}
The bias of DATD3 is larger than that of TD3, i.e., \rm{bias}($\mathcal{T}_{\mathrm{DATD3}}$) $\ge$ \rm{bias}($\mathcal{T}_{\mathrm{TD3}}$).
\end{theorem}

Theorem \ref{theo:daddpg} and Theorem \ref{theo:datd3} theoretically ensure the bias alleviation property of double actors, i.e., double actors architecture helps mitigate overestimation bias in DDPG as well as underestimation issues in TD3. To illustrate the bias alleviation effect with double actors, we conduct experiments in a typical MuJoCo \cite{todorov2012mujoco} environment, Walker2d-v2, where the value estimates are calculated by averaging over $1000$ states sampled from the replay buffer each tiemstep and the true value estimations are estimated by rolling out the current policy using the sampled states as the initial states and averaging the discounted long-term rewards. The experimental setting is identical as in Section \ref{sec:experiment} and the result is presented in Fig \ref{fig:biasperformance}. Fig \ref{fig:daddpgwalker2d} shows that DADDPG reduces the overestimation bias in DDPG and significantly outperforms DDPG in sample efficiency and final performance. As is shown in Fig \ref{fig:dacwalker2d}, DATD3 outperforms TD3 and preserves larger bias than TD3, which reveals the effectiveness and advantages of utilizing double actors to correct value estimation. The estimation bias comparison on broader environments can be found in Appendix \ref{sec:biaswithdoubleactors}.

\begin{figure}
    \centering
    \subfigure[Single critic network]{
    \label{fig:daddpgwalker2d}
    \includegraphics[scale=0.3]{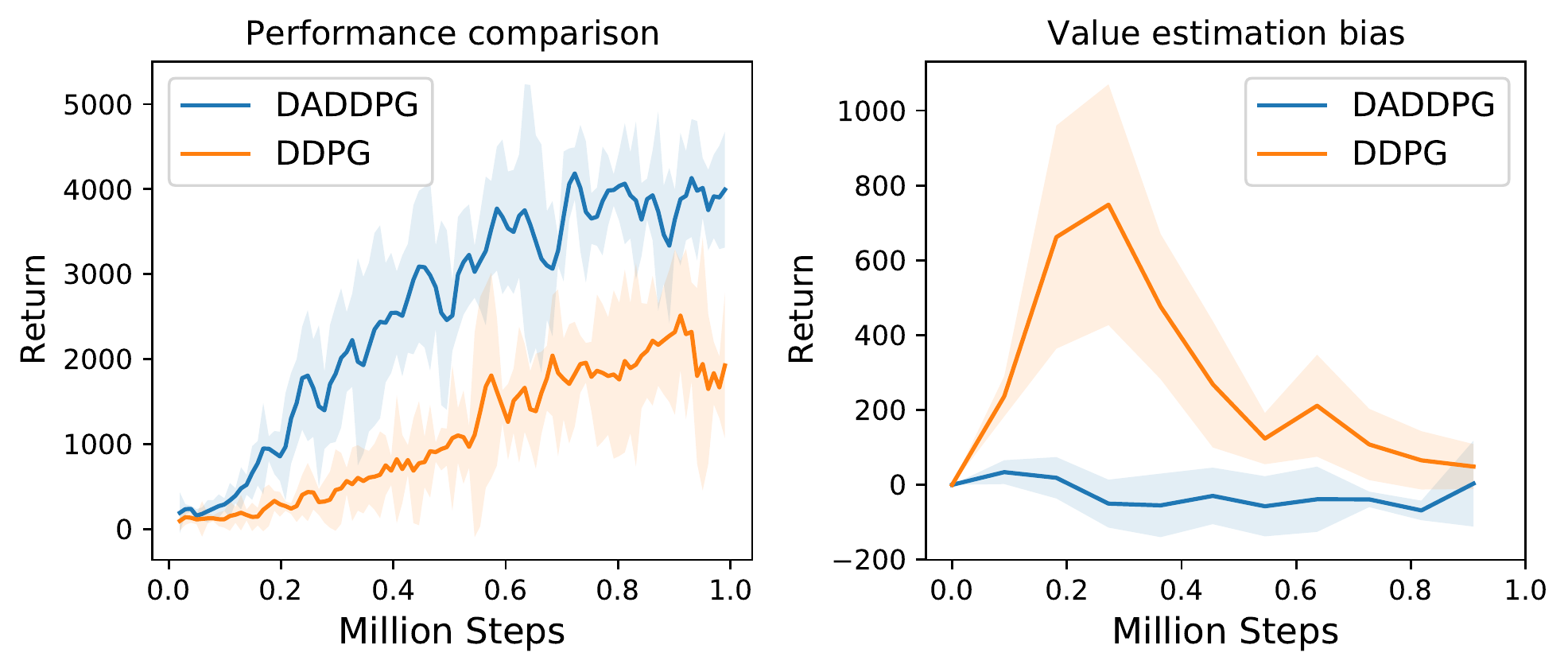}
    }\hspace{-2mm}
    \subfigure[Double critic networks]{
    \label{fig:dacwalker2d}
    \includegraphics[scale=0.3]{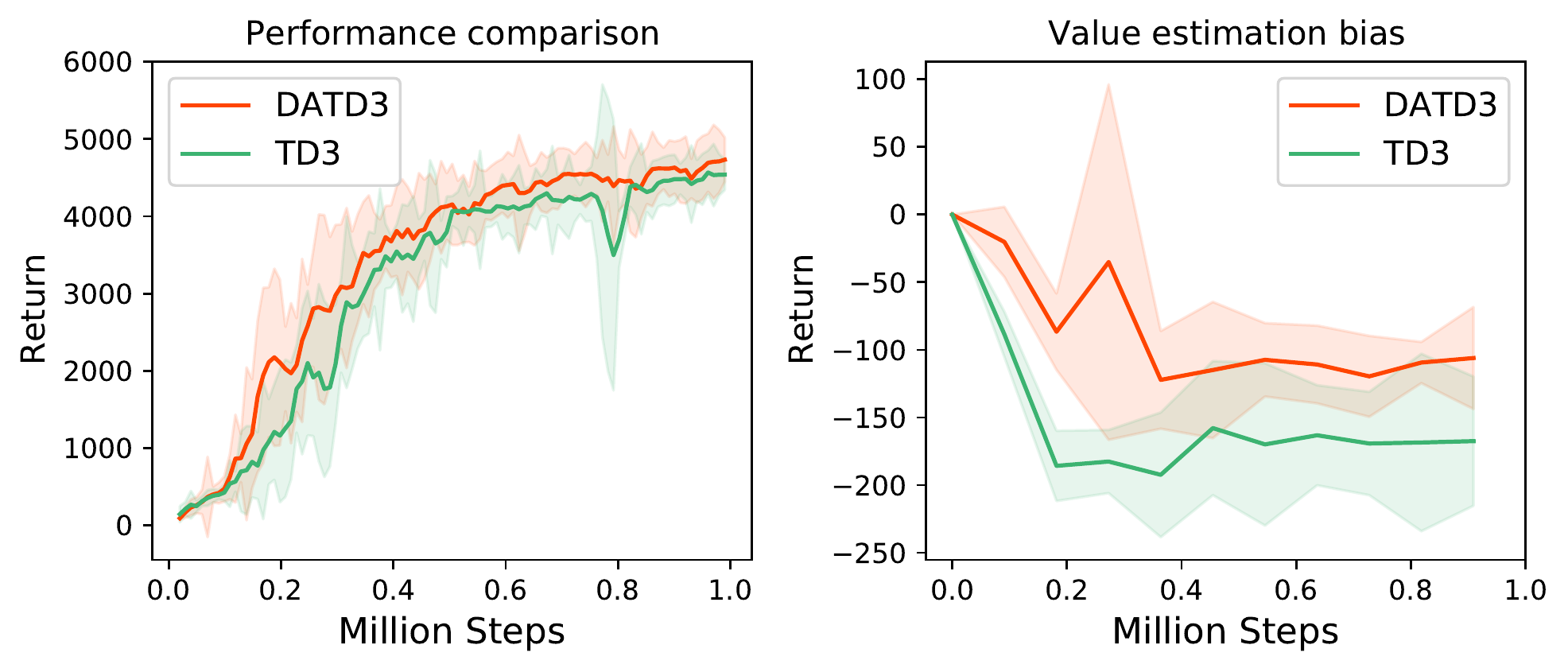}
    }\hspace{-2mm}
    \caption{Comparison of performance and value estimation bias on Walker2d-v2. Double actors help (a) relieve the overestimation bias in DDPG; (b) mitigate the underestimation bias in TD3.}
    \label{fig:biasperformance}
\end{figure}

\subsection{Enhanced Exploration with Double Actors}
\label{sec:exploration}

\begin{wrapfigure}{r}{4cm}
\centering
\label{fig:escapelocal}
\includegraphics[width=0.28\textwidth]{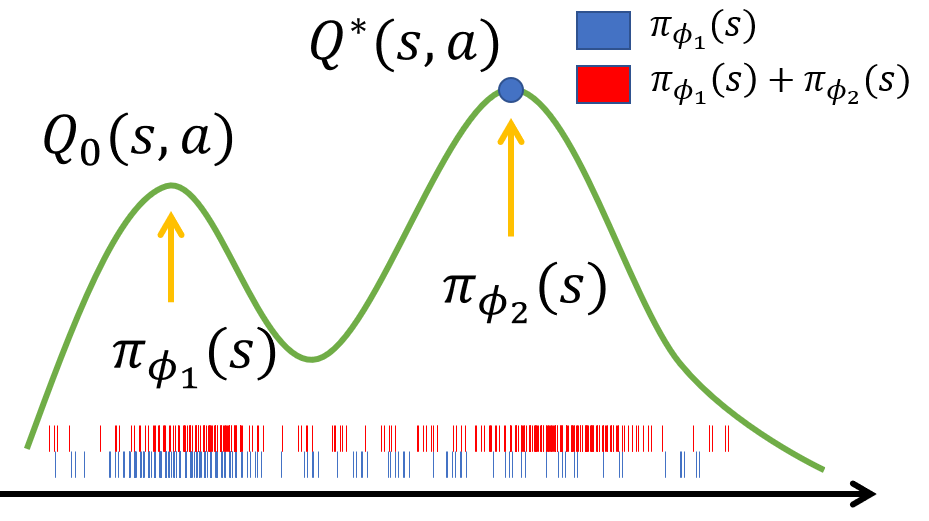}
\caption{Double actors help escape from local optimum.}
\end{wrapfigure}

Double actors allows the agent to evaluate different policies instead of being restricted by single policy path. Single actor $\pi_{\phi_1}(s)$ may make the agent stuck in a local optimum $Q_0(s,a)$ rather than the global optimum $Q^*(s,a)$ due to lack of exploration as is demonstrated in Fig 2. While double actors could enhance the \emph{exploration} ability of the agent by following the policy that results in higher return, i.e., the agent would follow $\pi_{\phi_1}(s)$ if $Q_\theta(s,\pi_{\phi_1}(s))\ge Q_\theta(s,\pi_{\phi_2}(s))$ and $\pi_{\phi_2}(s)$ otherwise. In this way, \emph{double actors can help escape from the local optimum} $Q_0(s,a)$ and reach the global optimum $Q^*(s,a)$.

We illustrate the exploration effect of double actors by designing an 1-dimensional, continuous state and action toy environment GoldMiner in Fig \ref{fig:goldminer} (see Appendix \ref{sec:doubleactorstoy} for detailed environmental setup). There are two gold mines centering at position $x_1=-3,x_2=4$ with neighboring region length to be $1$ where the miner can receive a reward of $+4$ and $+1$ if he digs in the right and left gold mine respectively, and a reward of $0$ if he digs elsewhere. The miner always starts at position $x_0=0$ and could move left or right to dig for gold with actions ranging from $[-1.5,1.5]$. The boundaries for the left and right side are $-4$ and $5$, and the episode length is $200$ steps. We run DDPG and DADDPG on GoldMiner for $40$ independent runs where we omit value correction using double actors to exclude its influence and only use the second actor for exploration in DADDPG for fair comparison. It can be found that DADDPG significantly outperforms DDPG as is shown in Fig \ref{fig:goldminerperformance} where the shaded region denotes one-third a standard deviation for better visibility.

\begin{figure}[htb]
    \centering
    \subfigure[GoldMiner]{
    \label{fig:goldminer}
    \includegraphics[scale=0.25]{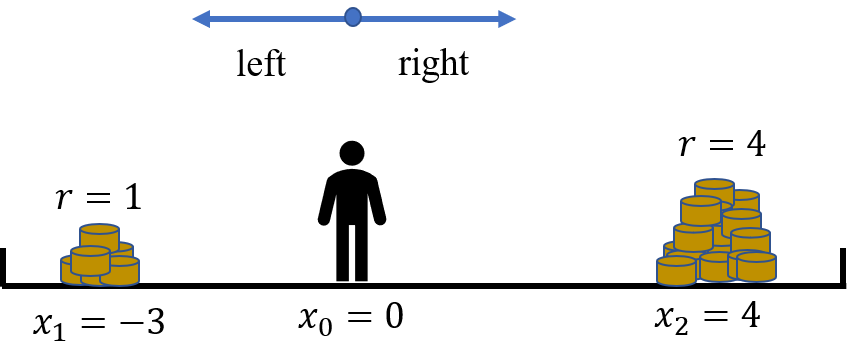}
    }\hspace{-2mm}
    \subfigure[Performance]{
    \label{fig:goldminerperformance}
    \includegraphics[scale=0.3]{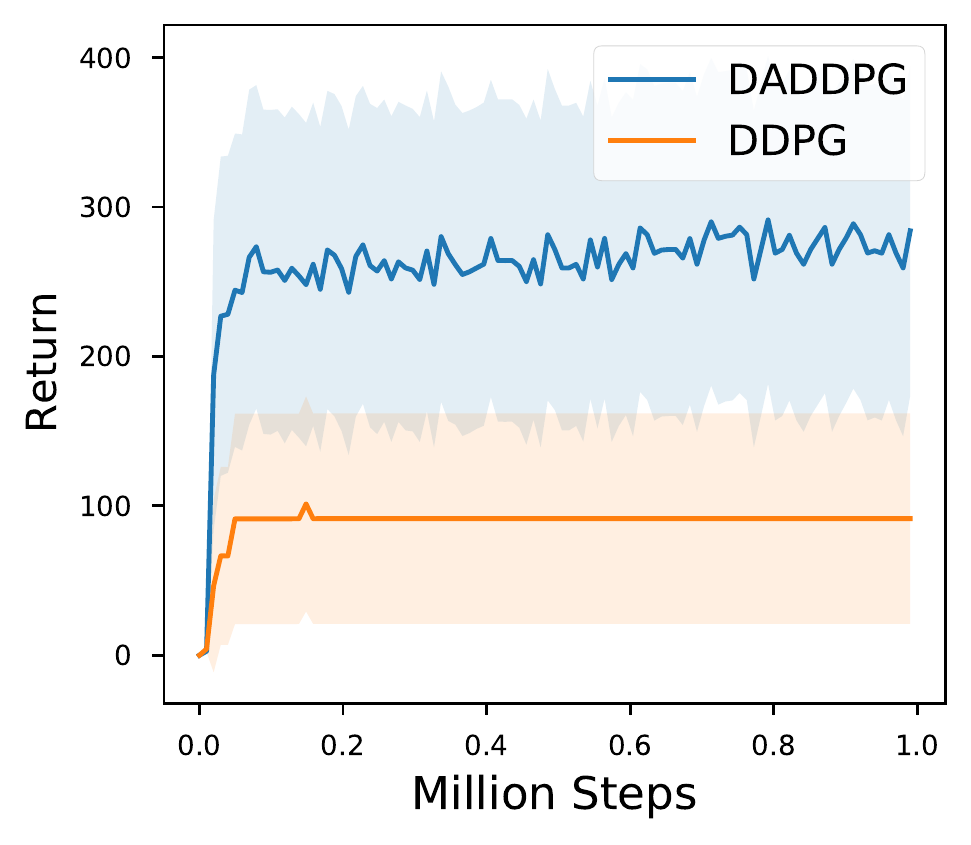}
    }\hspace{-2mm}
    \subfigure[State visit frequency]{
    \label{fig:statevisit}
    \includegraphics[scale=0.25]{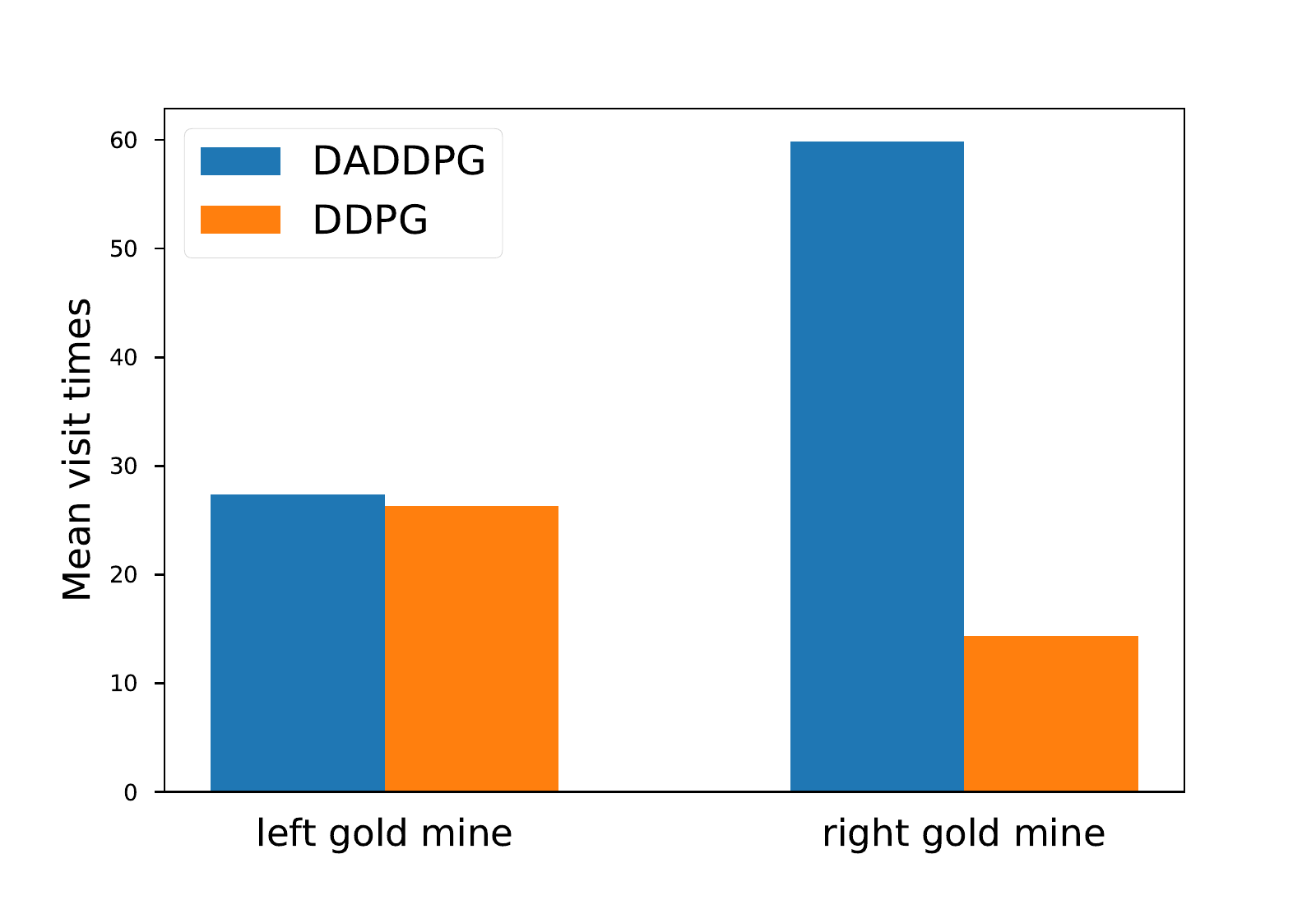}
    }\hspace{-2mm}
    \caption{Exploration ability analysis of double actors on GoldMiner environment.}
    \label{fig:exploration}
\end{figure}

To better understand the effectiveness of double actors, we collect the average high-reward state (where the gold mines lie) visiting times of each method for every episode. As is shown in Fig \ref{fig:statevisit}, the visiting frequency of DADDPG to the right gold mine which has larger reward significantly exceed that of DDPG, indicating that the agent would tend to visit the places that could achieve higher rewards with double actors. Single actor, however, would guide the agent to visit the left mine more frequently as it is closer, i.e., it is stuck in a local optimum. DDPG could hardly learn a policy towards the right gold mine, which is shortsighted and pessimistic (see Fig 2).

\begin{algorithm}[tb]
\caption{Double Actors Regularized Critics (DARC)}
\label{alg:algdarc}
\begin{algorithmic}[1] 
\STATE Initialize critic networks $Q_{\theta_1}, Q_{\theta_2}$ and actor networks $\pi_{\phi_1}, \pi_{\phi_2}$ with random parameters 
\STATE Initialize target networks $\theta_1^\prime \leftarrow \theta_1, \theta_2^\prime \leftarrow \theta_2, \phi_1^\prime \leftarrow \phi_1, \phi_2^\prime \leftarrow \phi_2$ and replay buffer $\mathcal{B} = \{\}$.
\FOR{$t$ = 1 to $T$}
\STATE Select action $a$ with Gaussian exploration noise $\epsilon$ based on $\pi_{\phi_1}$ and $\pi_{\phi_2}$, $\epsilon\sim \mathcal{N}(0,\sigma)$
\STATE Execute action $a$ and observe reward $r$, new state $s^\prime$ and done flag $d$
\STATE Store transitions in the replay buffer, i.e., $\mathcal{B}\leftarrow\mathcal{B}\bigcup \{(s,a,r,s^\prime,d)\}$
\FOR{$i = 1,2$}
\STATE Sample $N$ transitions $\{(s_j,a_j,r_j,s_j^\prime,d_j)\}_{j=1}^N\sim\mathcal{B}$
\STATE $a^\prime\leftarrow \pi_{\phi_1^\prime}(s^\prime) + \epsilon$, $a^{\prime\prime} \leftarrow \pi_{\phi_2^\prime}(s^\prime) + \epsilon$, $\epsilon\sim$ clip($\mathcal{N}(0,\bar{\sigma}),-c,c$)
\STATE $Q_1(s^\prime, a^\prime) \leftarrow \min_{j=1,2}\left(Q_{\theta_j^\prime}(s^\prime, a^\prime)\right)$, $Q_2(s^\prime, a^{\prime\prime}) \leftarrow \min_{k=1,2}\left(Q_{\theta_k^\prime}(s^\prime, a^{\prime\prime}) \right)$
\STATE $\hat{V}(s^\prime)\leftarrow \nu \min \{Q_1(s^\prime, a^\prime), Q_2(s^\prime, a^{\prime\prime})\} + (1-\nu) \max \{Q_1(s^\prime, a^\prime), Q_2(s^\prime, a^{\prime\prime})\}$
\STATE $y_t \leftarrow r + \gamma(1-d) \hat{V}(s^\prime)$
\STATE Update critic $\theta_i$ by minimizing: $\frac{1}{N}\sum_s \left \{ (Q_{\theta_i}(s,a)-y_t)^2 + \lambda [Q_{\theta_1}(s,a) - Q_{\theta_2}(s,a)]^2 \right \}$
\STATE Update actor $\phi_i$ with policy gradient: $ \frac{1}{N}\sum_s \nabla_a Q_{\theta_i}(s,a)|_{a=\pi_{\phi_i}(s)}\nabla_{\phi_i}\pi_{\phi_i}(s)$
\STATE Update target networks: $\theta_i^\prime \leftarrow \tau\theta_i + (1-\tau)\theta_i^\prime, \phi_i^\prime\leftarrow\tau\phi_i+(1-\tau)\phi_i^\prime$
\ENDFOR
\ENDFOR
\end{algorithmic}
\end{algorithm}

Moreover, double actors help relieve the \emph{pessimistic underexploration} phenomenon reported in \cite{ciosek2019better} upon double critics, which is caused by the reliance on pessimistic critics for exploration. This issue can be mitigated naturally with the aid of double actors as only the policy that leads to higher expected return would be executed, which is beneficial for exploration (see Appendix \ref{sec:pessimistic} for more details). With double actors, the exploration capability of the agent is decoupled from value estimation and \emph{the agent could benefit from both pessimistic estimation and optimistic exploration}. To conclude, double actors enable the agent to visit more valuable states and enhance the exploration capability of the agent, which makes the application of double actors in continuous control setting appealing.

\section{Beyond Double Critics: Double Actors with Regularized Critics}
\label{sec:darc}
Though DATD3 is qualified in target value correction, it does not address the core issue of the uncertainty in value estimation from two \emph{independent} critics. DATD3 may incur slight overestimation bias following Eq. (\ref{eq:datd3}) for value estimation (see the right figure in Fig \ref{fig:dacwalker2d}) due to the uncertainty over different critics. Moreover, it can observed in Fig \ref{fig:dacwalker2d} that the performance of DATD3 and TD3 are close, showing that the estimation bias control of DATD3 is limited. Thus, we analyze the estimation error introduced by double critics and propose two techniques to overcome it in this section. 

\textbf{Soft target value}. To control the underestimation bias more flexible, we propose to use a convex combination of $Q_1(s^\prime,\pi_{\phi_1^\prime}(s^\prime))$ and $Q_2(s^\prime,\pi_{\phi_2^\prime}(s^\prime))$ defined in Section \ref{sec:doubleactorvaluecorrection} for a softer estimation, which is given by:
\begin{equation}
    \label{eq:convexcomb}
    \hat{V}(s^\prime;\nu) = \nu \min \{ Q_1(s^\prime,\pi_{\phi_1^\prime}(s^\prime)), Q_2(s^\prime,\pi_{\phi_2^\prime}(s^\prime))\}+(1-\nu)\max \{Q_1(s^\prime,\pi_{\phi_1^\prime}(s^\prime), Q_2(s^\prime,\pi_{\phi_2^\prime}(s^\prime))\},
\end{equation}
where $\nu \in \mathbb{R}$ and $\nu\in[0,1)$. Eq. (\ref{eq:datd3}) is a special case of Eq. (\ref{eq:convexcomb}), and $\hat{V}(s^\prime;\nu)$ would lean towards the maximal value between $Q_1(s^\prime,\pi_{\phi_1^\prime}(s^\prime))$ and $Q_2(s^\prime,\pi_{\phi_2^\prime}(s^\prime))$ with $\nu\rightarrow 0$ and vice versa if $\nu\rightarrow 1$, in which case underestimation issues may arise. Therefore, we ought not to use large $\nu$ for training effective agents and there exist a trade-off for $\nu$, which is discussed in detail in Section \ref{sec:ablation}.

\textbf{Critic regularization.} We further propose to constrain the value estimates of the critics to mitigate the pessimistic value estimation, which leads to solving the optimization problem:
\begin{equation}
    \label{eq:optimizationproblem}
    \min_{\theta_i} \mathbb{E}_{s,a\sim\rho}[(Q_{\theta_i}(s,a)-y_t)^2], \quad \mathrm{s.t.} \quad \mathbb{E}_{s,a\sim\rho}[Q_{\theta_1}(s,a) - Q_{\theta_2}(s,a)] \le \delta,
\end{equation}
where $i\in\{1,2\}$ and $y_t=r_t+\gamma \hat{V}_t(s^\prime;\nu)$ is the target value. Note that there is no guarantee in better value estimate using merely convex combination because TD3 may induce underestimation problem as it utilizes the minimal value estimation from two \emph{independent} critic networks for target value update, where issues may arise if one of them fails to approximate the true Q-value, which is the root of the underestimation issue. With the soft estimation of value function via Eq. (\ref{eq:convexcomb}) that correlates both critics, we might mitigate the problem, while if both critics over- or underestimate the value function, the combined value estimate would be far from satisfying. The independence in critic networks increases the uncertainty and impedes the critics from estimating value function in a good manner. Regularizing critics such that their value estimate would not deviate from each other too far is a good way to mitigate the uncertainty. 

\textbf{Theoretical analysis.} We give an upper bound on value estimation using Eq. (\ref{eq:convexcomb}) based on the following definitions to further illustrate the necessity of critic regularization.

\begin{definition}
\label{def:valueerror}
Define the \textbf{value error} as the maximal distance between Eq. (\ref{eq:convexcomb}) with $\nu=0$ and the max operator at the $t$-th iteration, i.e., $\|\max_{a\in\mathcal{A}} Q_t(s,a) - \hat{V}_t(s;\nu=0)\|_\infty$.
\end{definition}

\begin{definition}
\label{def:policyerror}
Define the \textbf{policy execution error} as the maximal difference of value functions induced by the critic $Q_{\theta_i}$, $i=1,2$, that uses different policies $\pi_{\phi_1}(s)$ and $\pi_{\phi_2}(s)$, i.e., $\|Q_{\theta_i}(s,\pi_{\phi_1}(s)) - Q_{\theta_i}(s,\pi_{\phi_2}(s))\|_\infty$.
\end{definition}

\begin{definition}
\label{def:estimationerror}
Define the \textbf{critic deviance error} as the maximal distance of value functions induced by two critics under the same action $a$, i.e., $\|Q_{\theta_1}(s,a) - Q_{\theta_2}(s,a)\|_\infty$.
\end{definition}

If the errors defined above can be controlled during the training process, we could arrive at the following bound (see Appendix \ref{sec:theorem3} for proof):

\begin{theorem}
\label{theo:valueiteration}
Assume that the value error is bounded at the t-th iteration, i.e., $\|\max_{a\in\mathcal{A}} Q_t(s,a) - \hat{V}_t(s;\nu=0)\|_\infty \le \epsilon_t$ where $\epsilon_t$ is a time-variant upper bound. Assume that the policy execution error is bounded, i.e., $\|Q_{\theta_i}(s,\pi_{\phi_1}(s)) - Q_{\theta_i}(s,\pi_{\phi_2}(s))\|_\infty\le\epsilon_\pi, i=1,2$, and the critic deviance error is also bounded, i.e., $\forall s,a, \|Q_{\theta_1}(s,a) - Q_{\theta_2}(s,a)\|_\infty\le \epsilon_d$. Then for any iteration $t$, the difference between the optimal value function $V^*(s)$ and the value function $V(s)$ induced by the double actors satisfies:
\begin{equation*}
    \| V_t(s) - V^*(s)\|_\infty \le \gamma^t\|V_0(s) - V^*(s)\|_\infty + \sum_{k=0}^t \gamma^k \epsilon_k +  \dfrac{\nu}{1-\gamma}(2\epsilon_d+\epsilon_\pi).
\end{equation*}
\end{theorem}

\textbf{Remark}: This theorem guarantees the rationality of employing the value estimation induced by the double actors. The upper bound would be $O(\frac{\nu}{1-\gamma})$ if the value error can be controlled in a valid scale. It is reasonable to assume that the value error can be bounded by $\epsilon_t$ as the maximal estimation from two critics would fail to approximate the max operator if this requirement is violated, which would lead to large underestimation bias. Note that we cannot get a tighter bound by simply setting $\nu=0$ as these bounds are correlated with each other, e.g., the value error may become larger and lead to inaccuracy in value estimate if there are no constraints on policy execution error and critic deviance error due to the uncertainty in value estimation from two independent critics.

The value error is usually easy to control as the critic network is intrinsically trained to minimize this error to approach the true Q-value, while there are no guarantee in effective control in critic deviance error and the policy execution error as the critics and actors are trained independently, which increases the value estimation uncertainty and may negatively affect the performance of the agent. Therefore, it is necessary to regularize critics. In our implementation, we resort to penalty function methods by regularizing the original objective with critic deviance error to avoid the complex and expensive nonlinear programming costs. We also adopt a cross update scheme (see graphical illustration of this scheme in Appendix \ref{sec:cross-update}) where only one actor-critic pair is updated each time and meanwhile the other pair is only used for value correction, which naturally leads to the delayed update of the target network and hence contributes to the policy smoothing. We execute the action that brings higher return in our implementation to enjoy better exploration capability. The resulting method, Double Actors Regularized Critics (DARC) algorithm, is presented in Algorithm \ref{alg:algdarc}.

It is worth noting that the underestimation bias can be relieved with critic regularization naturally as the critic is constrained to approach both the target value and the other critic simultaneously. The estimation bias of DARC would be more conservative than that of DATD3, and DARC is more stable and efficient with soft value estimate and critic regularization. We present detailed bias comparison of DARC with DATD3 and TD3 in Appendix \ref{sec:biasdarc}.

\section{Experiments}
\label{sec:experiment}

In this section, we first conduct a detailed ablation study on DARC to investigate what contributes most to the performance improvement. We then extensively evaluate our method on two continuous control benchmarks where we use the state-of-the-art methods including TD3 \cite{fujimoto2018addressing} and SAC \cite{haarnoja2018soft} as baselines. Moreover, we extensively compare DARC with other value estimation correction methods on the benchmarks to further illustrate the effectiveness of DARC.

We adopt two widely-used continuous control benchmarks, OpenAI Gym \cite{brockman2016openai} simulated by MuJoCo \cite{todorov2012mujoco} and Box2d \cite{catto2011box2d} and PyBullet Gym simulated by PyBullet \cite{benelot2018}. We compare our method against DDPG, TD3 and Soft Actor-Critic (SAC), where we use the fine-tuned version of DDPG proposed in TD3 and temperature auto-tuned SAC \cite{haarnoja2018softactorcritic}. We use open-sourced implementation of these baselines \cite{TD3, SAC}. Each algorithm is repeated with $5$ independent seeds and evaluated for $10$ times every $5000$ timesteps. DARC shares the identical network configuration with TD3. The regularization coefficient is set to be $0.005$ by default and the value estimation weights $\nu$ is mainly selected from $[0.1,0.25]$ with $0.01$ as interval by using grid search. We use the same hyperparameters in DARC as the default setting for TD3 on all tasks except Humanoid-v2 where all these methods fail with default hyperparameters. We run for $3\times 10^6$ timesteps for Humanoid-v2 and $1\times10^6$ timesteps for the rest of the tasks for better illustration. Details for hyperparameters are listed in Appendix \ref{sec:hyperparameters}.

\subsection{Ablation study}
\label{sec:ablation}
We conduct the ablation study and parameter study on one typical MuJoCo environment HalfCheetah-v2, which is adequate to show the influence of different parts and parameters.

\begin{figure}
    \centering
    \subfigure[]{
    \label{fig:darccomponent}
    \includegraphics[scale=0.34]{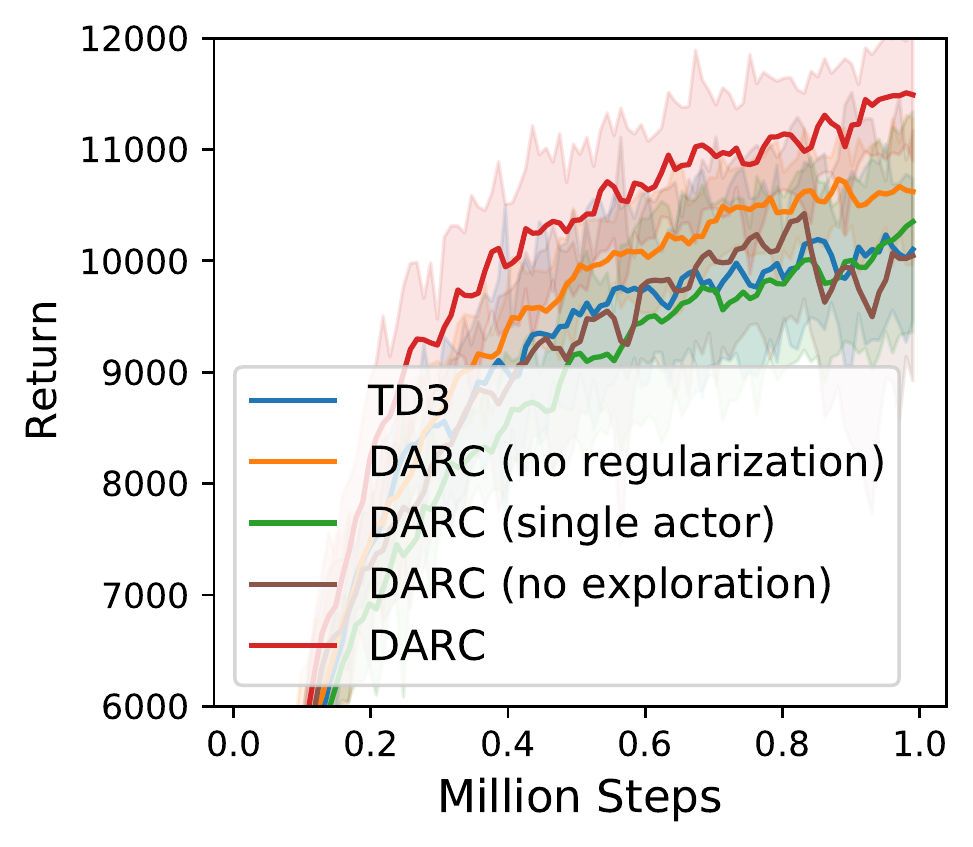}
    }\hspace{-2mm}
    \subfigure[]{
    \label{fig:darcupdate}
    \includegraphics[scale=0.34]{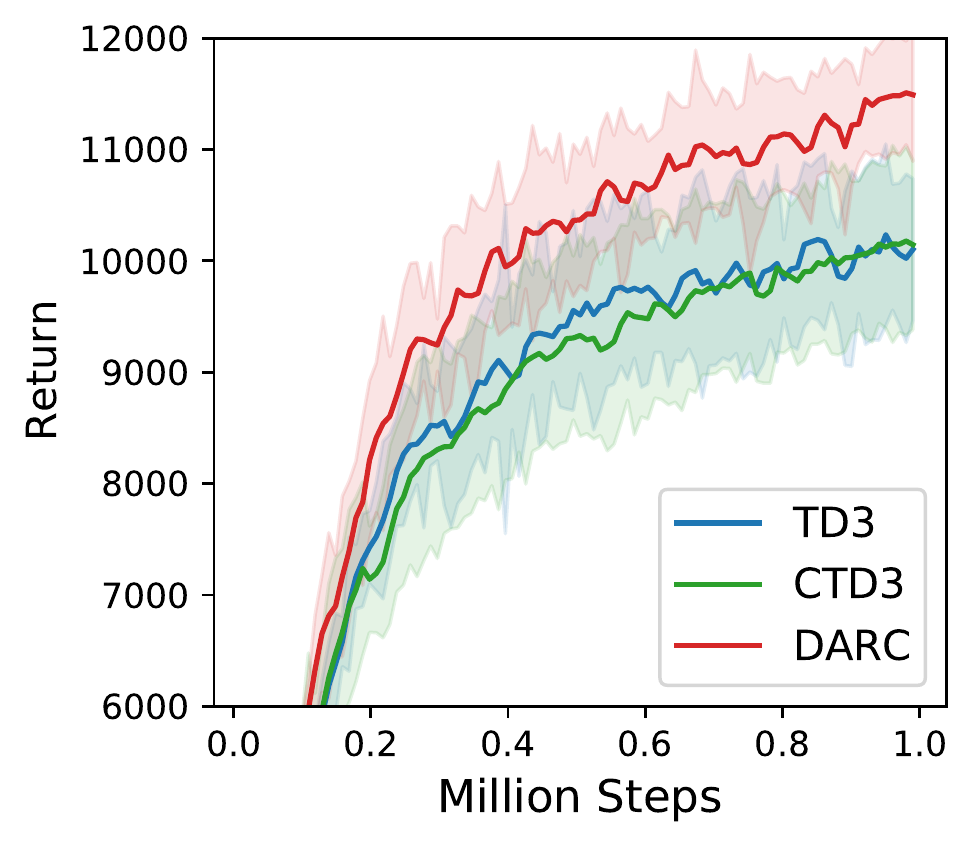}
    }\hspace{-2mm}
    \subfigure[]{
    \label{fig:darclambda}
    \includegraphics[scale=0.34]{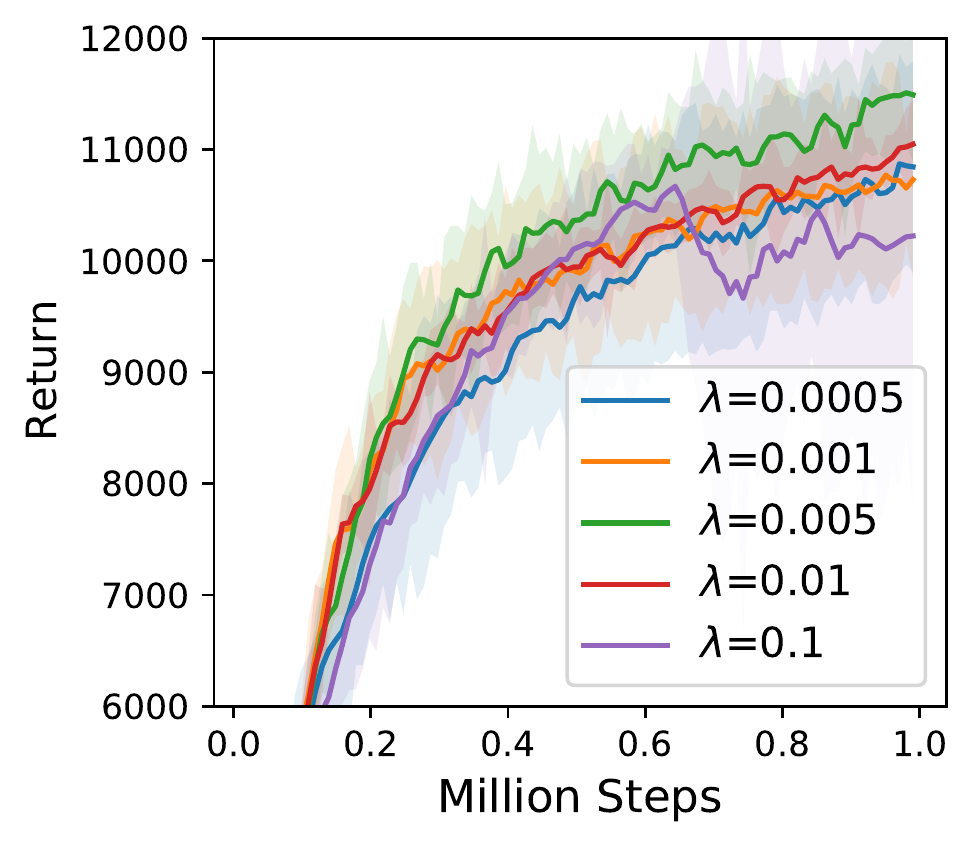}
    }\hspace{-2mm}
    \subfigure[]{
    \label{fig:darcnu}
    \includegraphics[scale=0.34]{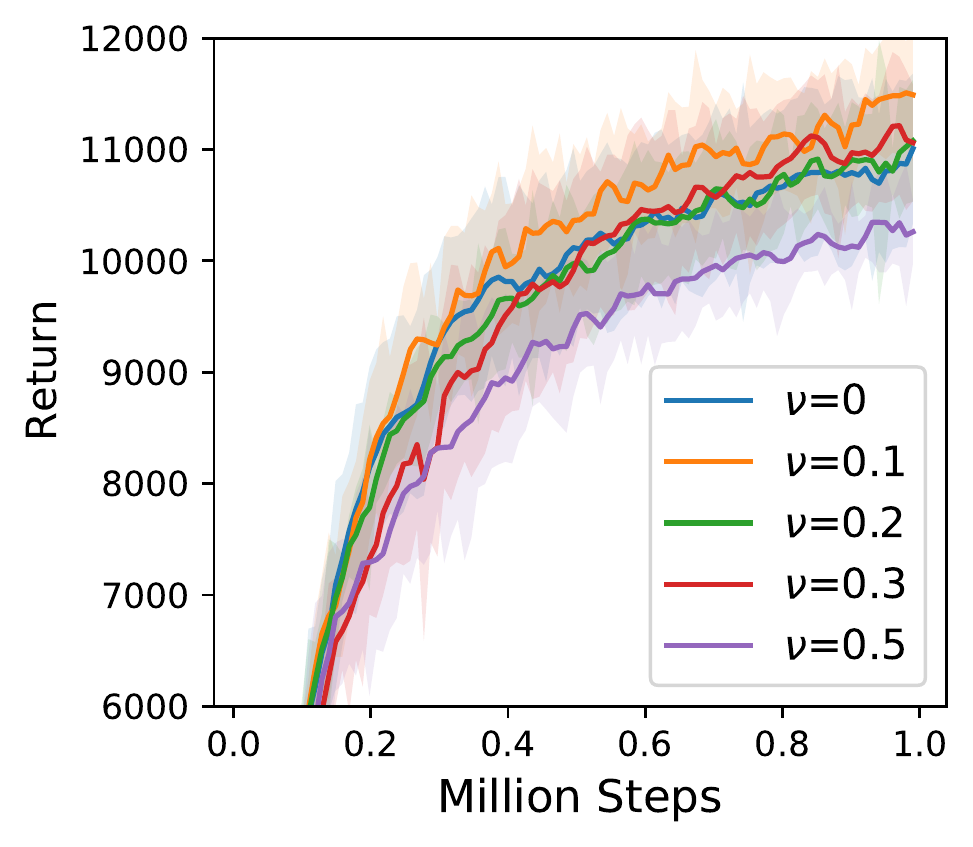}
    }\hspace{-2mm}
    \caption{Ablation study on HalfCheetah-v2 (5 runs, mean $\pm$ standard deviation). (a) Components; (b) Update scheme; (c) Regularization parameter $\lambda$; (d) Weighting coefficient $\nu$.}
    \label{fig:ablation}
\end{figure}

\textbf{Components}. We show in Fig \ref{fig:darccomponent} that DARC without either value estimation correction with double actors or regularization on critics would lead to a decrease in the performance on HalfCheetah-v2. We find that value correction with double actors contributes most to the performance improvement on TD3 while pure critic regularization can only be powerful if the value estimation is good enough. Furthermore, we exclude the exploration effect by executing actions following the first actor in DARC (see Fig \ref{fig:darccomponent}), which induces a decrease in performance. We also compare DARC with Cross-update TD3 (CTD3) that adopts cross update scheme where the critics are updated separately, and the result in Fig \ref{fig:darcupdate} shows that DARC outperfroms CTD3. The behavior of CTD3 is similar to that of vanilla TD3 as there does not exist any guarantee in better value estimation if we merely individually update critic network. DARC adopts cross update scheme to fulfill the delayed update in actor network.

\textbf{The regularization parameter} $\lambda$. $\lambda$ balances the influence of the difference in two critic networks. Large $\lambda$ may cause instability in value estimation and impede the agent from learning a good policy. While small $\lambda$ may induce a slow and conservative update, which weakens the effect and benefits of critic regularization. Luckily, there does exist an intermediate value that could achieve a trade-off as is shown in Fig \ref{fig:darclambda}, where one can find that DARC is not sensitive to $\lambda$ as long as it is not too large. Note that $0.005$ may not be the optimal $\lambda$ for all tasks, and hence there is every possibility that one may get even better performance than the results reported in this paper by tuning this parameter.

\textbf{The weighting coefficient} $\nu$. The weighting coefficient $\nu$ directly influences the value estimation of DARC. As is discussed in Section \ref{sec:darc}, large $\nu$ would yield underestimation issues and small $\nu$ may induce large overestimation bias. We show in Fig \ref{fig:darcnu} that there exists a suitable $\nu$ that could offer the best trade-off. Note that the coefficient $\nu$ is diverse for different tasks.

\subsection{Extensive experiments}
\label{sec:extensiveexp}

The overall performance comparison is presented in Fig \ref{fig:darcresult} where the solid line represents the averaged return and the shaded region denotes standard deviation. We use the smoothing strategy with sliding window $3$ that is suggested in OpenAI baselines \cite{baselines} for better demonstration. As is demonstrated in Fig \ref{fig:darcresult}, DARC significantly outperforms TD3 with much higher sample efficiency, e.g., DARC consumes 50\% fewer interaction times to reach the highest return than TD3 in HalfCheetah-v2 task with around 30\% additional training time. DARC learns much faster than other methods. 

\begin{figure}
    \centering
    \includegraphics[scale=0.35]{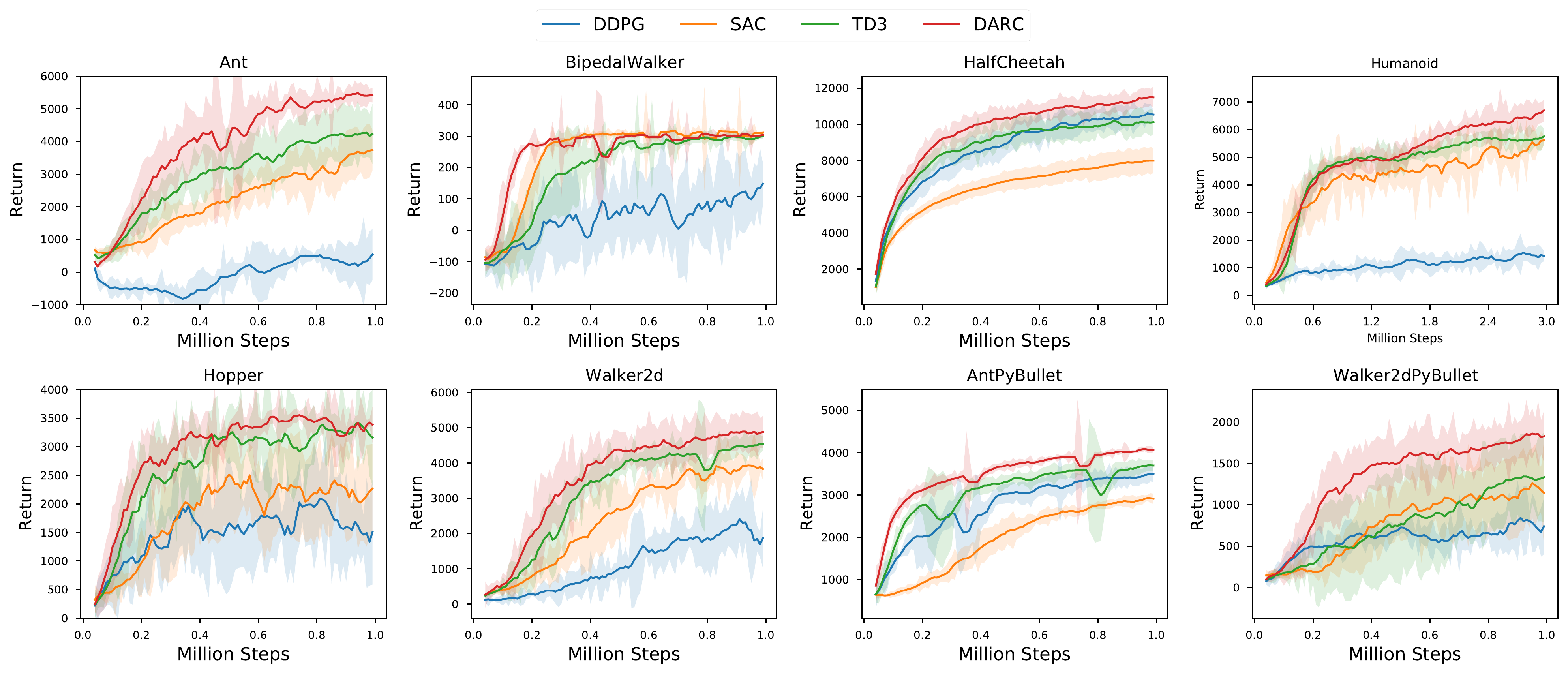}
    \caption{Performance comparison in OpenAI Gym and PyBullet Gym environments.}
    \label{fig:darcresult}
\end{figure}

\subsection{Comparison with other value correction methods}

We additionally compare DARC with other recent value correction methods, SD3 \cite{pan2020softmax} and TADD \cite{wu2020reducing}, where SD3 leverages softmax operator on value function for a softer estimation and TADD leverages triple critics by weighting over them for a better estimation. We also compare DARC with Double Actors TD3 (DATD3). We conduct numerous experiments in identical environments in Section \ref{sec:extensiveexp} and report the final mean score over 5 independent runs in Table \ref{tab:scorecompare}, where one could see that DARC significantly outperforms these value correction methods in all tasks.

\section{Related Work}

Actor-Critic methods \cite{konda2000actor, prokhorov1997adaptive, konda1999actor} are widely-used in Reinforcement Learning (RL). The quality of the learned critic network is vital for a good performance in an RL agent when applying function approximation \cite{barth-maron2018distributional}, e.g., we can get an unbiased estimate of policy gradient if we enforce the critic to meet the compatibility conditions \cite{silver2014deterministic}.

How to estimate the value function in a good way remains an important problem in RL, and has been widely investigated in deep Q-network (DQN) \cite{hasselt2016deep, mnih2015human, sabry2019reduction} for discrete regime control. Lan et al. \cite{lan2020minmax} propose to take the minimum Q-value estimation under the ensemble scheme to control the value estimation bias in DQN, while Anschel et al. \cite{anschel2017averaged} leverage the average value of an ensemble of Q-networks for variance reduction. Apart from Q-ensemble methods, many effective and promising methods involving estimation weighting \cite{zhang2017weighted}, softmax operator \cite{song2019revisiting} are also explored. 

In continuous control domain, DDPG suffers from large overestimation bias and the improvement built upon DDPG includes distributional \cite{barth2018distributed, bellemare2017distributional}, model-based \cite{feinberg2018model,gu2016continuous}, prioritized experience replay \cite{horgan2018distributed} method, etc. TD3 tackles the issue by using double critics for value correction, while it may suffer from severe underestimation problem. There are many efforts in utilizing TD3 for distributional training \cite{ma2020dsac}, Hierarchical RL (HRL) \cite{nachum2018data}, and so on, while value estimation correction methods for performance improvement are rarely investigated.  Wu et al. \cite{wu2020reducing} adopt triple critics and correct the value estimation by weighting over these critics, while Pan et al. \cite{pan2020softmax} propose to apply softmax operator for value correction. There are also some prior works \cite{kuznetsov2020controlling, roy2020opac} that adopt critic ensemble for bias alleviation, while they both underperform SD3. Also, two parallel actor-critic architecture are explored to learn better options \cite{ZhangW19dac}. Despite these advances, few of them investigate the role and benefits of double actors in value correction, which is the focus of our work. Furthermore, training multiple critic or actor networks can be expensive, while DARC is efficient.

\begin{table}
  \caption{Numerical performance comparison on final score (3M steps for Humanoid and 1M steps for the rest) between DARC and other value estimation correction methods. W2dPyBullet refers to Walker2dPybullet. The best results are in bold.}
  \label{tab:scorecompare}
  \centering
  \begin{tabular}{lllllll}
    \toprule
    Environment   & TD3 & TADD & SD3 & DATD3 (ours) &  \textbf{DARC (ours)} \\
    \midrule
    Ant  & 4164.10 & 4593.01 & 4541.71 & 5180.29 & \textbf{5642.33}$\pm$\textbf{188.82}   \\
    BipedalWalker  & 294.08 & 303.42 & 299.69 & 305.09 &  \textbf{311.25}$\pm$\textbf{2.66}  \\
    HalfCheetah  & 10237.62 & 10099.80 & 10934.72 & 10623.96 &  \textbf{11600.74}$\pm$\textbf{499.11} \\
    Hopper & 3145.20 & 3142.16 & 3286.24 & 2822.94 &  \textbf{3577.93}$\pm$\textbf{133.97} \\
    Humanoid & 5992.28 & 6182.54 & 5809.18 & 5960.03 &  \textbf{6737.63}$\pm$\textbf{743.95} \\
    Walker2d & 4605.25 & 4834.59 & 4622.89 & 4694.75 &  \textbf{5045.36}$\pm$\textbf{548.12} \\
    AntPyBullet & 3683.49 & 3216.75 & 3762.93 & 3949.02 &  \textbf{4100.01}$\pm$\textbf{19.24} \\
    W2dPyBullet & 1385.01 & 1150.07 & 1497.06 & 1777.24 &  \textbf{1902.46}$\pm$\textbf{217.25}\\
    \bottomrule
  \end{tabular}
\end{table}

Finally, our method is related to the regularization method, which has been broadly used outside RL, for instance, machine learning \cite{bauer2007regularization, poggio1987computational}, computer vision \cite{girosi1995regularization, moradi2020survey, wan2013regularization}, etc. Inside RL, regularization strategy is widely used in offline RL \cite{lange2012batch, wu2019behavior}, model-based RL \cite{boney2020regularizing, d2020learn}, and maximum entropy RL \cite{haarnoja2018soft, zhao2019maximum}. We, however, propose to regularize critics to ensure that the value estimation from them would not deviate far from each other, which contributes to the robustness of the DARC algorithm.

\section{Conclusion}

In this paper, we explore and illustrate the benefits of double actors in continuous control tasks, which has long been ignored, where we show the preeminent exploration property and the bias alleviation property of double actors on both single critic and double critics. We further propose to regularize critics to enjoy better stability and mitigate large difference in value estimation from two independent critics. Taken together, we present Double Actors Regularized Critics (DARC) algorithm which extensively and significantly outperforms state-of-the-art methods as well as other value estimation correction methods on standard continuous control benchmarks. 

The major limitation of our work is the additional actor network training cost due to double actor structure, which would consume more time to train the agent than TD3 given the same interaction steps. While this cost is tolerable considering the increasing computation speed as well as capacity and the high sample efficiency of DARC. For future work, it would be interesting to extend the DARC algorithm from double-actor-double-critic architecture into multi-actor-multi-critic structure.

\small
\bibliographystyle{abbrv}
\bibliography{neurips_2021.bib}

\appendix

\section{Double Actors for Value Correction}

\subsection{Double Actors on Single Critic}
\label{sec:doubleactoronecritic}
\subsubsection{The DADDPG Algorithm}
The full algorithm of DADDPG (Double Actors Deep Deterministic Policy Gradient) is shown in Algorithm \ref{alg:algdaddpg}. DADDPG is constructed based on Deep Deterministic Policy Gradient (DDPG) \cite{lillicrap2015continuous} algorithm where we use double actors to mitigate the devastating overestimation issues that is often reported in DDPG. As is in TD3, DADDPG adopts delayed update in the target network for better stability in policy execution. The actors in DADDPG are not updated at the same time such that both actors are updated in a delayed style, which is similar in spirit to the cross-update scheme in DATD3 and DARC. DADDPG is a simple yet effective variant of DDPG.

\begin{algorithm}[htb]
\caption{Double Actors Deep Deterministic Policy Gradient (DADDPG)}
\label{alg:algdaddpg}
\begin{algorithmic}[1] 
\STATE Initialize critic networks $Q_{\theta}$ and actor networks $\pi_{\phi_1}, \pi_{\phi_2}$ with random parameters $\theta$, $\phi_1$, $\phi_2$
\STATE Initialize target networks $\theta^\prime \leftarrow \theta, \phi_1^\prime \leftarrow \phi_1, \phi_2^\prime \leftarrow \phi_2$ and replay buffer $\mathcal{B} = \{\}$.
\FOR{$t$ = 1 to $T$}
\STATE Select action $a$ with Gaussian exploration noise $\epsilon$ based on $\pi_{\phi_1}$ and $\pi_{\phi_2}$, $\epsilon\sim \mathcal{N}(0,\sigma)$
\STATE Execute action $a$ and observe reward $r$, new state $s^\prime$ and done flag $d$
\STATE Store transitions in the replay buffer, i.e., $\mathcal{B}\leftarrow\mathcal{B}\bigcup \{(s,a,r,s^\prime,d)\}$
\STATE Sample $N$ transitions $\{(s_j,a_j,r_j,s_j^\prime,d_j)\}_{j=1}^N\sim\mathcal{B}$
\STATE $a^\prime\leftarrow \pi_{\phi_1^\prime}(s^\prime) + \epsilon$, $a^{\prime\prime} \leftarrow \pi_{\phi_2^\prime}(s^\prime) + \epsilon$, $\epsilon\sim$ clip($\mathcal{N}(0,\bar{\sigma}),-c,c$)
\STATE $\hat{V}(s^\prime)\leftarrow \min \left\{ Q_\theta(s^\prime, a^\prime), Q_\theta(s^\prime,a^{\prime\prime})  \right\}$
\STATE $y_t \leftarrow r + \gamma(1-d) \hat{V}(s^\prime)$
\STATE Update critic $\theta$ by minimizing: $\frac{1}{N}\sum_s  (Q_{\theta}(s,a)-y_t)^2$
\IF{$t \mod 2$}
\STATE Update actor $\phi_1$ with policy gradient: $ \frac{1}{N}\sum_s \nabla_a Q_{\theta}(s,a)|_{a=\pi_{\phi_1}(s)}\nabla_{\phi_1}\pi_{\phi_1}(s)$
\STATE Update target networks: $\theta^\prime \leftarrow \tau\theta + (1-\tau)\theta^\prime, \phi_1^\prime\leftarrow\tau\phi_1+(1-\tau)\phi_1^\prime$
\ELSE
\STATE Update actor $\phi_2$ with policy gradient: $ \frac{1}{N}\sum_s \nabla_a Q_{\theta}(s,a)|_{a=\pi_{\phi_2}(s)}\nabla_{\phi_2}\pi_{\phi_2}(s)$
\STATE Update target network: $\phi_2^\prime\leftarrow\tau\phi_2+(1-\tau)\phi_2^\prime$
\ENDIF
\ENDFOR
\end{algorithmic}
\end{algorithm}

\subsubsection{Comparison of DADDPG and TD3}
\label{sec:daddpgtd3}
DADDPG leverages double actors for value correction while TD3 leverages double critics for value correction \cite{fujimoto2018addressing}. As is in TD3, we adopt delayed update in target networks and the update of double actors are crossed instead of updated simultaneously. TD3 uses the first critic to update the actor network while we utilize the actor network that results in larger return for action execution. 

Naturally, one may be interested in the following question: which way of correction is better? We answer this question by comparing DADDPG with TD3 in three commonly used environments in MuJoCo \cite{todorov2012mujoco}, Ant-v2, Hopper-v2 and Walker2d-v2. The result is shown in Fig \ref{fig:daddpgvstd3}, where one could find that DADDPG underperforms TD3 while both DADDPG and TD3 outperform vanilla DDPG.

DADDPG outperforms DDPG naturally due to the value correction benefits and improved exploration ability with the aid of double actors, where the severe overestimation bias of DDPG is effectively reduced. While DADDPG behave not as good as TD3 does, because the actors may not execute satisfying policies if the critic network itself is badly fitted. An actor is as good as is allowed by its counterpart critic. We extensively compare DADDPG and TD3 on broader environments in MuJoCo \cite{todorov2012mujoco} and Pybullet \cite{benelot2018} (1M steps and 5 runs for each method) as is shown in Table \ref{tab:correctionmethodcomp} where one may find that DADDPG outperforms TD3 on some tasks like Walker2dPybullet. While on most of the tasks, DADDPG is worse than TD3. Moreover, the good performance of DADDPG is credit to the advanced exploration capability brought by double actors and one can find that DADDPG without exploration (i.e., we only adopt the second actor for value correction and rely on the first actor for policy execution) performs much worse than that of TD3. Hence, the value estimation correction with double critics is better than that with double actors, and that is one major reason that we choose to correct the value estimation in double actors double critics architecture by taking minimum firstly over double critics with respect to each actor network and then take maximal value of value estimations from double actors for final value function estimate.

\begin{figure}
    \centering
    \includegraphics[scale=0.5]{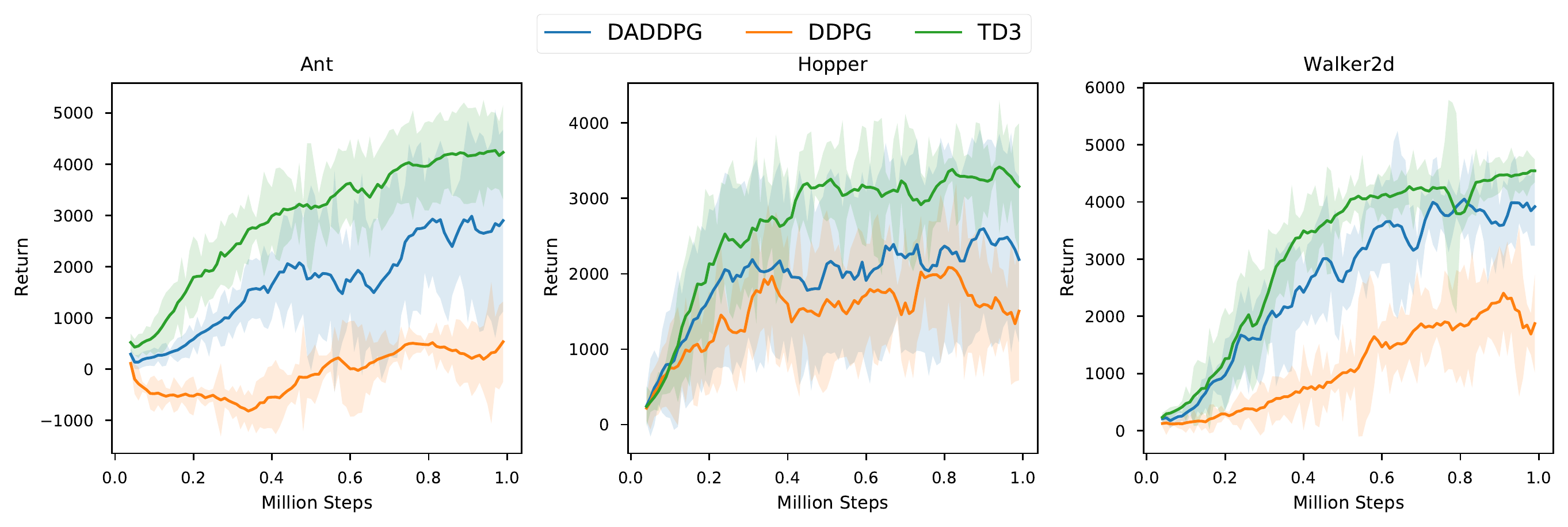}
    \caption{Performance comparison of DDPG, DADDPG and TD3 on typical MuJoCo environments (5 runs, mean $\pm$ standard deviation).}
    \label{fig:daddpgvstd3}
\end{figure}

\begin{table}
  \caption{Numerical performance comparison on final score of DDPG, DADDPG (no exploration), DADDPG and TD3. W2dPyBullet refers to Walker2dPybullet. The bese results are in bold.}
  \label{tab:correctionmethodcomp}
  \centering
  \begin{tabular}{lllllll}
    \toprule
    Environment   & DDPG & DADDPG (no exploration) & DADDPG (ours) & TD3 \\
    \midrule
    Ant  & 537.67 & 1013.44 & 2977.07 &  \textbf{4164.10}  \\
    BipedalWalker & 182.16 & 207.47 & 219.61 & \textbf{294.08} \\
    HalfCheetah  & 10561.40 & 10236.91 & \textbf{10644.77} & 10237.62 \\
    Hopper & 1770.31 & 2946.32 & 2284.35 & \textbf{3145.20} \\
    Walker2d & 1384.86 & 4370.71 & 3939.56 & \textbf{4605.25}  \\
    AntPyBullet & 3435.11 & 3677.42  & \textbf{3743.63} & 3683.49 \\
    W2dPyBullet & 838.59 & 1120.52 & \textbf{1889.49} & 1385.01  \\
    \bottomrule
  \end{tabular}
\end{table}

\subsection{Double Actors on Toy GoldMiner Environment}
\label{sec:doubleactorstoy}
The detailed hyperparameter setup of DDPG and DADDPG for the GoldMiner toy environment is listed in Table \ref{tab:goldminer} which is similar to the parameter setup in Table \ref{tab:hyperparameter}. We run both DDPG and DADDPG for $40$ times independently with random seed ranging from $0$ to $39$. Since this GoldMiner environment is a sparse reward environment (as the miner could only get positive rewards when he digs in left or right gold mine which only has length of $1$), it is hard for vanilla DDPG to learn useful policy due to its inaccuracy in value estimation and weakness in exploration, while DADDPG could do better in exploring high-return states with the aid of double actors and hence learn better policies. DADDPG adopts the policy that would lead to higher expected return, which inherently encourages its exploration in unknown environments, and the value correction with the aid of double actors offer an guarantee in better value estimation. Double actors offer the agent two policy paths to select from where the agent can explore different paths instead of being restricted with a single policy as one can never know whether this single policy is good enough or not. We use one-third a standard deviation for better readability because of the large variance in the return curve.

\begin{table*}
\centering
\caption{Hyperparameters setup for DDPG and DADDPG on GoldMiner environment}
\label{tab:goldminer}
\begin{tabular}{lrr}
\toprule
\textbf{Hyperparameter}  & \textbf{Value} \\
\midrule
Shared & \\
\qquad Actor network  & \qquad  $(400,300)$ \\
\qquad Critic network & \qquad  $(400,300)$ \\
\qquad Batch size     &\qquad   $100$ \\
\qquad Learning rate  & \qquad $10^{-3}$ \\
\qquad Optimizer & \qquad Adam \\
\qquad Discount factor & \qquad $0.99$ \\
\qquad Replay buffer size & \qquad $10^6$  \\
\qquad Warmup steps & \qquad $10^4$ \\
\qquad Exploration noise &\qquad  $\mathcal{N}(0,0.1)$ \\
\qquad Target update rate & \qquad $5\times 10^{-3}$ \\
\midrule
DADDPG  & \\
\qquad Noise clip &\qquad  $0.5$ \\
\qquad Target noise &\qquad  $0.2$ \\
\bottomrule
\end{tabular}
\end{table*}

\subsubsection{Double Actors Help Mitigate the Pessimistic Underexploration Issues}
\label{sec:pessimistic}

\emph{Pessimistic underexploration} phenomenon is firstly introduced in \cite{ciosek2019better} which describes the insufficient exploration of the agent when interacting with the environment, which is caused by the lower bound approximation to the critic in methods like TD3 and SAC. Issues may arise if the critic is inaccurate and the maximum of the lower bound is spurious because pessimistic estimate of the critic along with a greedy actor update impedes the agent from executing unknown actions and exploring unknown states. However, this issue can be well mitigated with the aid of double actors. We validate this by conducting experiments on toy GoldMiner environment where we run TD3 and DATD3 $40$ times independently with random seeds $0$-$39$ where the second actor in DATD3 is only used for exploration for fair comparison, i.e., no value correction based on double actors is applied. The detailed hyperparameter setup is presented in Table \ref{tab:goldminertd3}. The performance comparison is available in Fig \ref{fig:toyperformance}. One could find that DATD3 significantly outperforms TD3 in both sample efficiency and overall performance. In Section \ref{sec:exploration} of the main text, we show that double actors improve value estimation upon single critic and could result in better exploration in the toy GoldMiner environment, where DADDPG visits the right gold mine with reward $+4$ more frequently, indicating that double actors on single critic could help escape from the local optimum (see Fig \ref{fig:statevisit} in the main text). We show in Fig \ref{fig:toyvisittime} that DATD3 also successfully converges to the global optimum by visiting the right gold mine (with reward $+4$) more frequently and visiting the left gold mine with reward $+1$ significantly less than TD3. The double critics in TD3 are pessimistic while the underexploration issue can be mitigated as double actors encourage exploration of the agent by executing the action that would bring higher future return, leading to better exploration to unknown states and actions. Therefore, we conclude that double actors effectively mitigate the pessimistic underexploration phenomenon in TD3, which naturally results in better performance on continuous control tasks.

\begin{figure}
    \centering
    \subfigure[Performance comparison]{
    \label{fig:toyperformance}
    \includegraphics[scale=0.5]{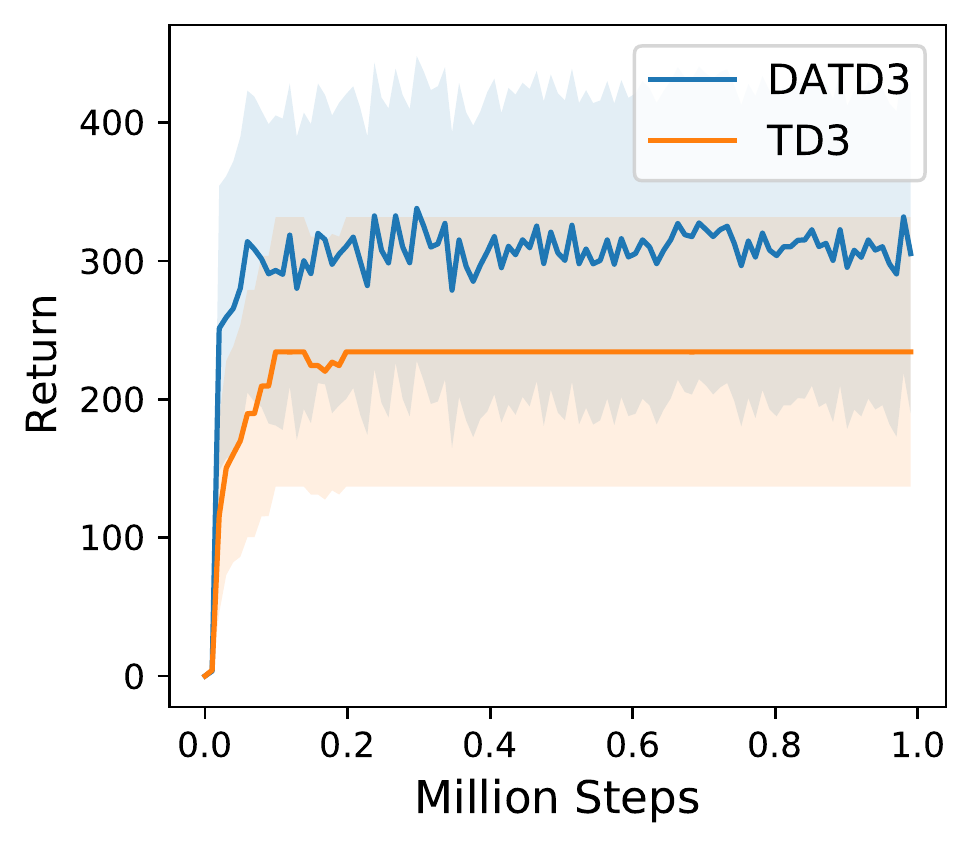}
    }\hspace{-2mm}
    \subfigure[Valuable State visit frequency]{
    \label{fig:toyvisittime}
    \includegraphics[scale=0.47]{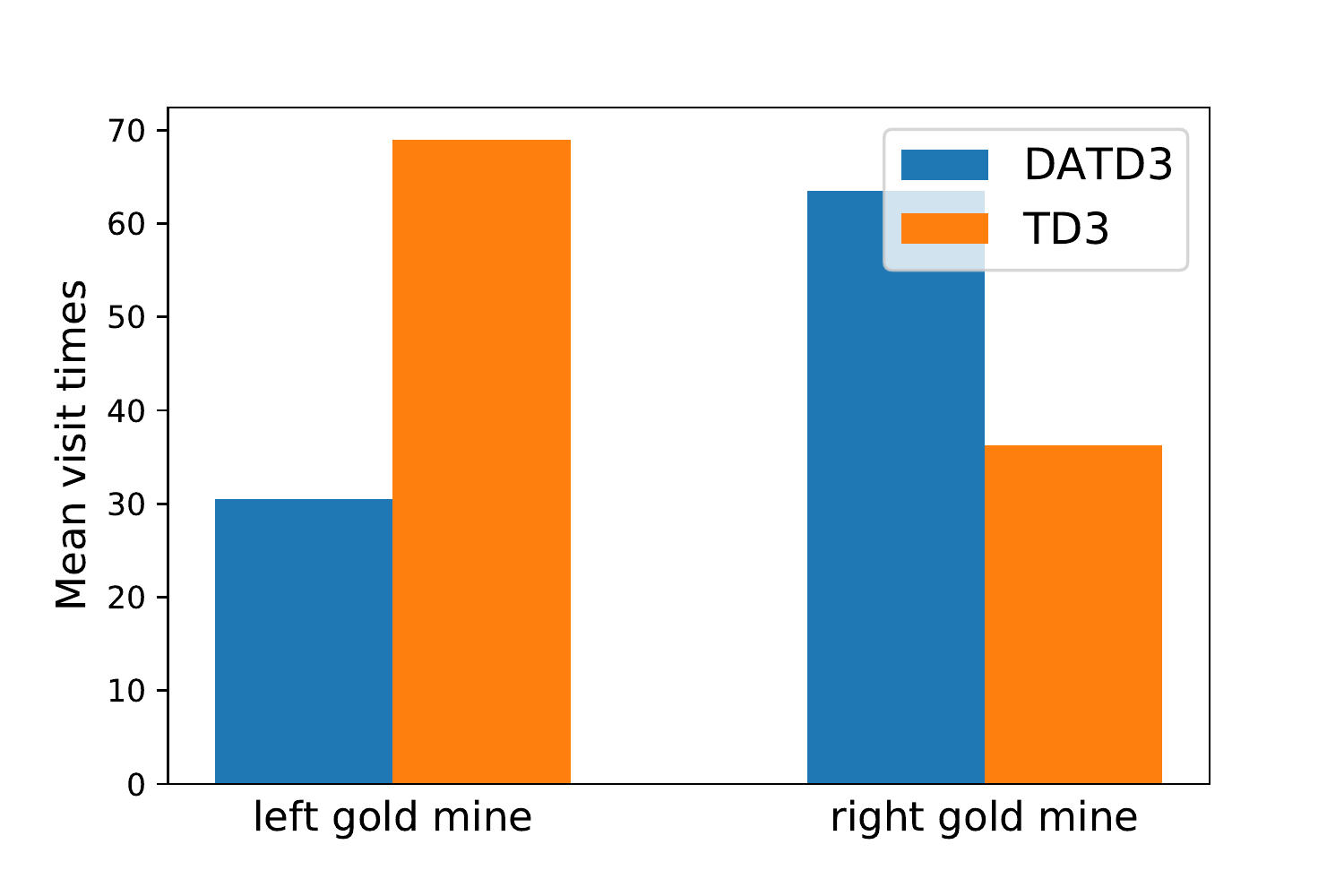}
    }\hspace{-2mm}
    \caption{Performance comparison and valuable state visit frequency comparison of TD3 and DATD3 on toy GoldMiner environment.}
    \label{fig:goldminerdoubleactor}
\end{figure}

\begin{table*}
\centering
\caption{Hyperparameters setup for TD3 and DATD3 on GoldMiner environment}
\label{tab:goldminertd3}
\begin{tabular}{lrr}
\toprule
\textbf{Hyperparameter}  & \textbf{Value} \quad \\
\midrule
Shared & \\
\qquad Actor network  & \qquad  $(400,300)$ \\
\qquad Critic network & \qquad  $(400,300)$ \\
\qquad Batch size     &\qquad   $100$ \\
\qquad Learning rate  & \qquad $10^{-3}$ \\
\qquad Optimizer & \qquad Adam \\
\qquad Discount factor & \qquad $0.99$ \\
\qquad Replay buffer size & \qquad $10^6$  \\
\qquad Warmup steps & \qquad $10^4$ \\
\qquad Exploration noise &\qquad  $\mathcal{N}(0,0.1)$ \\
\qquad Target update rate & \qquad $5\times 10^{-3}$ \\
\qquad Target noise & \qquad $0.2$ \\
\qquad Noise clip & \qquad $0.5$ \\
\midrule
TD3 & \\
\qquad Target update interval & $2$ \\
\bottomrule
\end{tabular}
\end{table*}

\subsection{Double Actors on Double Critics}
\label{sec:doubleactorsdoublecritics}
\subsubsection{Graphical Illustration on difference between TD3 and DARC}
\label{fig:graphdifference}
In this part, we provide a detailed graphical comparison on structures of TD3 and DARC which is available in Fig \ref{fig:structuretd3}. One could find that the second actor in TD3 is merely used for value correction and is wasted as the actor does not depend on it for policy update. While the critics in DARC are fully utilized where each critic is in charge of the corresponding counterpart actor and both critics are responsible for policy improvement. The critics in DARC are not isolated, but are connected with each other, i.e., they are regularized to be close and are used for value estimation correction for each other.

\begin{figure}
    \centering
    \includegraphics[scale=0.5]{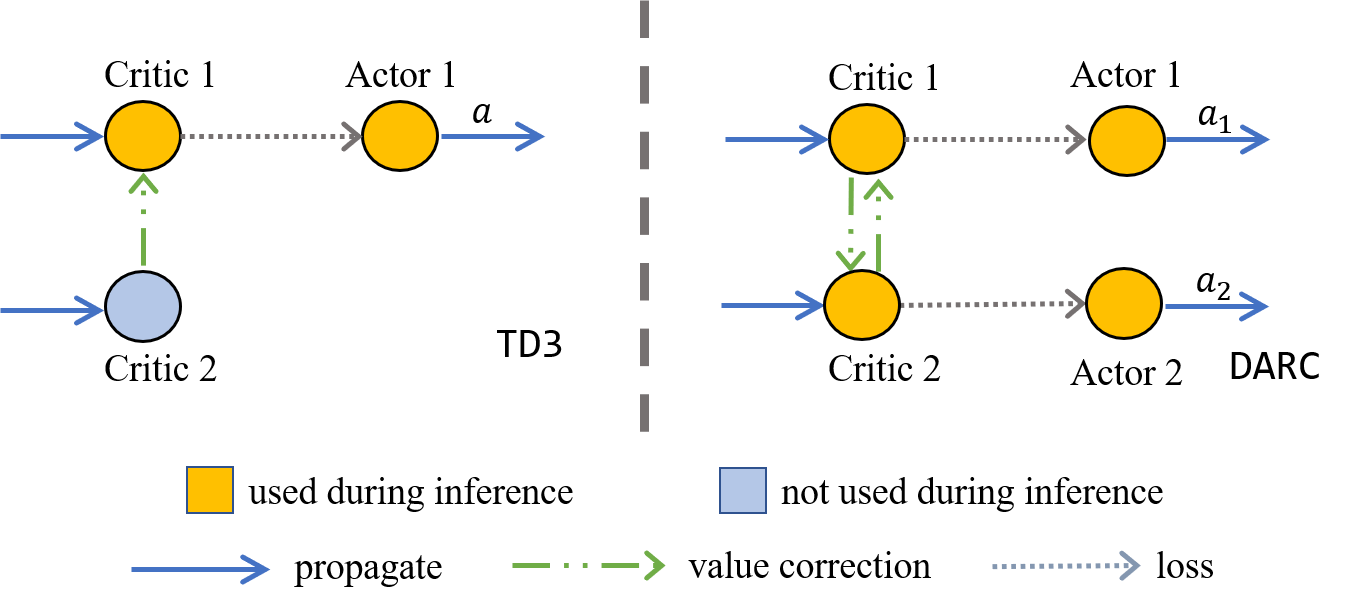}
    \caption{Structure comparison of TD3 and DARC.}
    \label{fig:structuretd3}
\end{figure}

\subsubsection{The DATD3 Algorithm}
The Double Actors Twin Delayed Deep Deterministic Policy Gradient (DATD3) algorithms consists of two actor networks and two critic networks where we leverage the double actors for value function correction so as to relieve the underestimation bias in TD3. We also encourage exploration of the agent by executing the policy that would induce higher expected future return so as to mitigate the severe \emph{pessimistic underexploration} problem in TD3 as is discussed in detail in Appendix \ref{sec:pessimistic}. The full algorithm of DATD3 is available in Algorithm \ref{alg:algdac}.

\begin{algorithm}[tb]
\caption{Double Actors Twin Delayed Deep Deterministic Policy Gradient (DATD3)}
\label{alg:algdac}
\begin{algorithmic}[1] 
\STATE Initialize critic networks $Q_{\theta_1}, Q_{\theta_2}$ and actor networks $\pi_{\phi_1}, \pi_{\phi_2}$ with random parameters $\theta_1$, $\theta_2$, $\phi_1$, $\phi_2$
\STATE Initialize target networks $\theta_1^\prime \leftarrow \theta_1, \theta_2^\prime \leftarrow \theta_2, \phi_1^\prime \leftarrow \phi_1, \phi_2^\prime \leftarrow \phi_2$ and replay buffer $\mathcal{B} = \{\}$.
\FOR{$t$ = 1 to $T$}
\STATE Select action $a$ with Gaussian exploration noise $\epsilon$ based on $\pi_{\phi_1}$ and $\pi_{\phi_2}$, $\epsilon\sim \mathcal{N}(0,\sigma)$
\STATE Execute action $a$ and observe reward $r$, new state $s^\prime$ and done flag $d$
\STATE Store transitions in the replay buffer, i.e. $\mathcal{B}\leftarrow\mathcal{B}\bigcup \{(s,a,r,s^\prime,d)\}$
\FOR{$i = 1,2$}
\STATE Sample $N$ transitions $\{(s_j,a_j,r_j,s_j^\prime,d_j)\}_{j=1}^N\sim\mathcal{B}$
\STATE $a^\prime\leftarrow \pi_{\phi_1^\prime}(s^\prime) + \epsilon$, $a^{\prime\prime} \leftarrow \pi_{\phi_2^\prime}(s^\prime) + \epsilon$, $\epsilon\sim$ clip($\mathcal{N}(0,\bar{\sigma}),-c,c$)
\STATE $Q_1(s^\prime, a^\prime) \leftarrow \min_{j=1,2}\left(Q_{\theta_j^\prime}(s^\prime, a^\prime)\right)$, $Q_2(s^\prime, a^{\prime\prime}) \leftarrow \min_{k=1,2}\left(Q_{\theta_k^\prime}(s^\prime, a^{\prime\prime}) \right)$
\STATE $\hat{V}(s^\prime)\leftarrow \max \{Q_1(s^\prime, a^\prime), Q_2(s^\prime, a^{\prime\prime})\}$
\STATE $y_t \leftarrow r + \gamma(1-d) \hat{V}(s^\prime)$
\STATE Update critic $\theta_i$ by minimizing: $\frac{1}{N}\sum_s (Q_{\theta_i}(s,a)-y_t)^2$
\STATE Update actor $\phi_i$ with policy gradient: $ \frac{1}{N}\sum_s \nabla_a Q_{\theta_i}(s,a)|_{a=\pi_{\phi_i}(s)}\nabla_{\phi_i}\pi_{\phi_i}(s)$
\STATE Update target networks: $\theta_i^\prime \leftarrow \tau\theta_i + (1-\tau)\theta_i^\prime, \phi_i^\prime\leftarrow\tau\phi_i+(1-\tau)\phi_i^\prime$
\ENDFOR
\ENDFOR
\end{algorithmic}
\end{algorithm}

\subsection{Numerical Comparison of DDPG, DADDPG, TD3, DATD3 and DARC}
\label{sec:numercomp}
In this part, we present the final mean return comparison of DDPG, DADDPG, TD3, DATD3 and DARC (3M steps for Humanoid-v2 and 1M steps for other environments) in Table \ref{tab:algcomp} to illustrate the effectiveness and advantages of utilizing double actors and critic regularization in continuous control tasks. The best results are in bold where one could find that DADDPG significantly outperforms DDPG. DADDPG is not as competitive as TD3 on many environments while it can outperform TD3 on tasks like AntPybullet, Walker2dPybullet, etc. DATD3 has similar performance as TD3 while DARC extensively outperforms all of other methods.

\begin{table}
  \caption{Numerical performance comparison on final score of DDPG, DADDPG, TD3, DATD3 and DARC. W2dPyBullet refers to Walker2dPybullet. The best results are in bold.}
  \label{tab:algcomp}
  \centering
  \begin{tabular}{lllllll}
    \toprule
    Environment   & DDPG & DADDPG (ours) & TD3 & DATD3 (ours) &  \textbf{DARC (ours)} \\
    \midrule
    Ant  & 537.67 & 2977.07 & 4164.10 & 5180.29 & \textbf{5642.33}$\pm$\textbf{188.82}   \\
    BipedalWalker & 182.16 & 219.61 & 294.08 & 305.09 &  \textbf{311.25}$\pm$\textbf{2.66}  \\
    HalfCheetah  & 10561.40 & 10644.77 & 10237.62 & 10623.96 &  \textbf{11600.74}$\pm$\textbf{499.11} \\
    Hopper & 1770.31 & 2284.35 & 3145.20 & 2822.94 &  \textbf{3577.93}$\pm$\textbf{133.97} \\
    Humanoid & 1329.62 & 4807.04 & 5992.28 & 5960.03 &  \textbf{6737.63}$\pm$\textbf{743.95} \\
    Walker2d & 1384.86 & 3939.56 & 4605.25 & 4694.75 &  \textbf{5045.36}$\pm$\textbf{548.12} \\
    AntPyBullet & 3435.11 & 3743.63  & 3683.49 & 3949.02 &  \textbf{4100.01}$\pm$\textbf{19.24} \\
    W2dPyBullet & 838.59 & 1889.49 & 1385.01 & 1777.24 &  \textbf{1902.46}$\pm$\textbf{217.25}\\
    \bottomrule
  \end{tabular}
\end{table}

\subsection{Graphical Illustration on Cross-update Scheme}
\label{sec:cross-update}
In this part, we present the graphical illustration on cross-update scheme that we adopt in DARC in Fig \ref{fig:crossupdate}. Each timestep, only one actor-critic pair is updated and the other pair is merely used for value correction. The first actor is driven by the first critic and the second critic is in charge of the second actor. By introducing such scheme, the delayed update in actors are fulfilled, which is beneficial to policy smoothing as is discussed in TD3 \cite{fujimoto2018addressing}.

\begin{figure}
    \centering
    \includegraphics[scale=0.4]{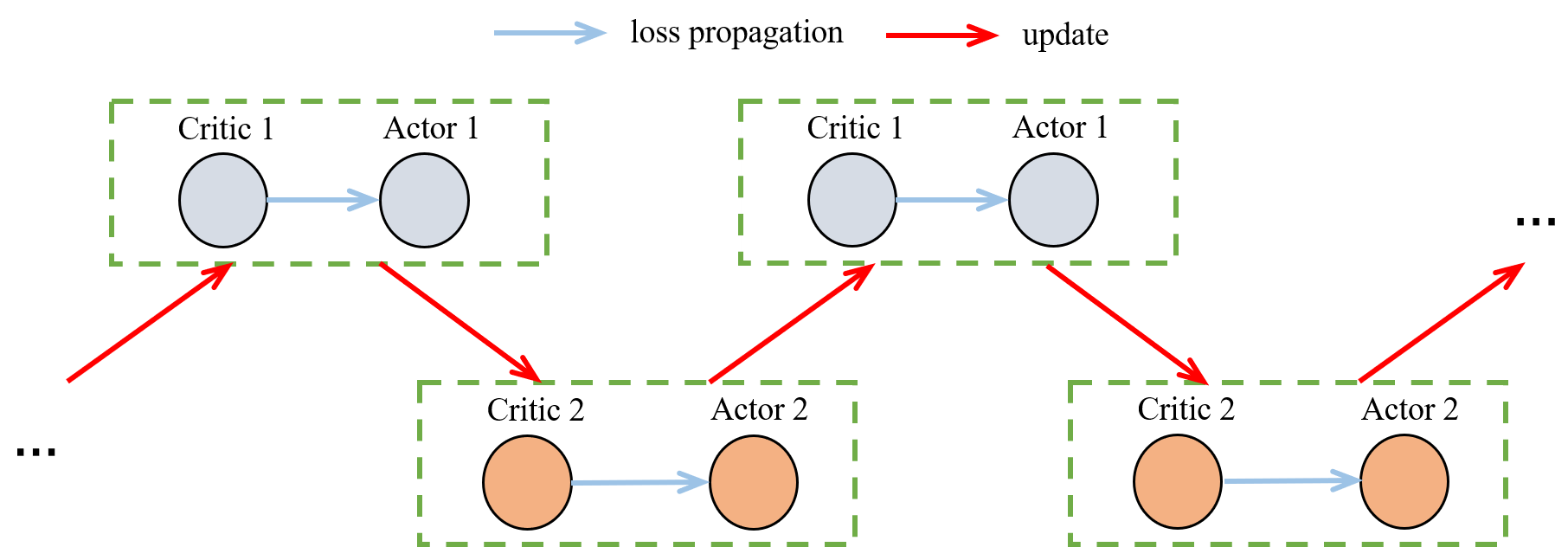}
    \caption{Graphical illustration on the cross-update scheme.}
    \label{fig:crossupdate}
\end{figure}

\section{Proofs in Section 3}
\subsection{Proof of Theorem 1}
\label{sec:theorem1}
\begin{theorem}
\label{appendixtheo:daddpg}
Denote the value estimation bias deviating the true value induced by $\mathcal{T}$ as bias($\mathcal{T}$) = $\mathbb{E}[\mathcal{T}(s^\prime)] - \mathbb{E}[Q_{\theta^{\rm{true}}}(s^\prime,\pi_{\phi^\prime}(s^\prime))]$, then we have \rm{bias}($\mathcal{T}_{\mathrm{DADDPG}}$) $\le$ \rm{bias}($\mathcal{T}_{\mathrm{DDPG}}$).
\end{theorem}

\begin{proof}
By definition, we have
\begin{equation}
    \mathcal{T}_{\mathrm{DDPG}}(s^\prime) = Q_{\theta^\prime}(s^\prime,\pi_{\phi_1}(s^\prime)),
\end{equation}
and,
\begin{equation}
    \mathcal{T}_{\mathrm{DADDPG}}(s^\prime) = \min \{ Q_{\theta^\prime}(s^\prime,\pi_{\phi_1^\prime}(s^\prime)), Q_{\theta^\prime}(s^\prime,\pi_{\phi_2^\prime}(s^\prime)) \}.
\end{equation} 

Obviously, we have 
\begin{equation}
    \mathcal{T}_{\mathrm{DADDPG}}(s^\prime) \le \mathcal{T}_{\mathrm{DDPG}}(s^\prime).
\end{equation}

Therefore, $\mathbb{E}[\mathcal{T}_{\mathrm{DDPG}}(s^\prime)]\ge\mathbb{E}[\mathcal{T}_{\mathrm{DADDPG}}(s^\prime)]$, and naturally, bias($\mathcal{T}_{\mathrm{DADDPG}}$) $\le$ bias($\mathcal{T}_{\mathrm{DDPG}}$).

\end{proof}

\noindent \textbf{Remark:} DADDPG is actually DDPG with dual-path, where the agent correct its value estimation according to the smaller one conservatively. Hence either path of DADDPG induces a DDPG value estimator. Like TD3, DADDPG is naturally designed to mitigate the overestimation bias in DDPG, which is also easy to implement. Note that there does not exist any assumptions or requirements in this theorem, which sheds light to the benefits of employing double actors for value estimation correction.

\subsection{Proof of Theorem 2}
\label{sec:theorem2}
\begin{theorem}
\label{appendixtheo:dac}
The bias of DATD3 is larger than that of TD3, i.e. \rm{bias}($\mathcal{T}_{\mathrm{DATD3}}$) $\ge$ \rm{bias}($\mathcal{T}_{\mathrm{TD3}}$).
\end{theorem}

\begin{proof}
By definition, we have
\begin{equation}
    \mathcal{T}_{\mathrm{TD3}}(s^\prime) = \min \{ Q_{\theta_1^\prime}(s^\prime,\pi_{\phi_1^\prime}(s^\prime)), Q_{\theta_2^\prime}(s^\prime,\pi_{\phi_1^\prime}(s^\prime)) \}.
\end{equation}

similarly, we have
\begin{equation}
    \mathcal{T}_{\mathrm{DATD3}}(s^\prime) = \max \left\{ \min_{i=1,2} Q_{\theta_i^\prime}(s^\prime,\pi_{\phi_1^\prime}(s^\prime)), \min_{j=1,2} Q_{\theta_j^\prime}(s^\prime,\pi_{\phi_2^\prime}(s^\prime)) \right\}
\end{equation}

Therefore $\mathcal{T}_{\mathrm{DATD3}}(s^\prime) \ge \mathcal{T}_{\mathrm{TD3}}(s^\prime)$, which leads to $\mathbb{E}[\mathcal{T}_{\mathrm{DATD3}}(s^\prime)] \ge \mathbb{E}[\mathcal{T}_{\mathrm{TD3}}(s^\prime)]$. Hence, bias($\mathcal{T}_{\mathrm{DATD3}}$) $\ge$ bias($\mathcal{T}_{\mathrm{TD3}}$).

\end{proof}

\noindent \textbf{Remark:} DATD3 consists of double actors and double critics and we correct the value estimates by taking minimum of double critics first and then take maximum value estimation from double actors and use that as final value estimate. The major reason that we correct the value estimation in this way lies in the fact that DADDPG without exploration significantly underperforms TD3 (see Appendix \ref{sec:daddpgtd3}), indicating that applying critics for value correction first is better than that of using double actors for value correction first. Furthermore, there does not exist any guarantee in preserving larger bias compared with TD3 if we take maximum value estimate over critics first, which may also introduce overestimation issues. For each actor in DATD3, it is actually a TD3-style structure where the actors share identical critic networks while only its counterpart critic is used for loss propagation. DATD3 naturally ease the underestimation bias in TD3 and hence can improve its performance on broad tasks.

\section{Bias Comparison on Broader Tasks}
\subsection{Bias Alleviation with Double Actors}
\label{sec:biaswithdoubleactors}
In this section, we include bias comparison of DADDPG and DDPG on broader MuJoCo environments, Ant-v2 and Hopper-v2. We also compare DATD3 and TD3 on these environment to show the generality of the bias alleviation property with double actors. Similarly, the value estimates are calculated by averaging over $1000$ states sampled from the replay buffer each tiemstep and the true value estimations are estimated by rolling out the current policy using the sampled states as the initial states and averaging the discounted long-term rewards. The experimental setting is identical as in Section \ref{sec:experiment} of the main text. The results are shown in Fig \ref{fig:biascomp} where one can see that double actors on single critic significantly relieve the overestimation bias in DDPG on both tasks (see Fig \ref{fig:daddpgbroader}) and double actors on double critics significantly relieve the underestimation bias in TD3 (see Fig \ref{fig:datd3broader}).

\begin{figure}
    \centering
    \subfigure[DADDPG vs DDPG]{
    \label{fig:daddpgbroader}
    \includegraphics[scale=0.32]{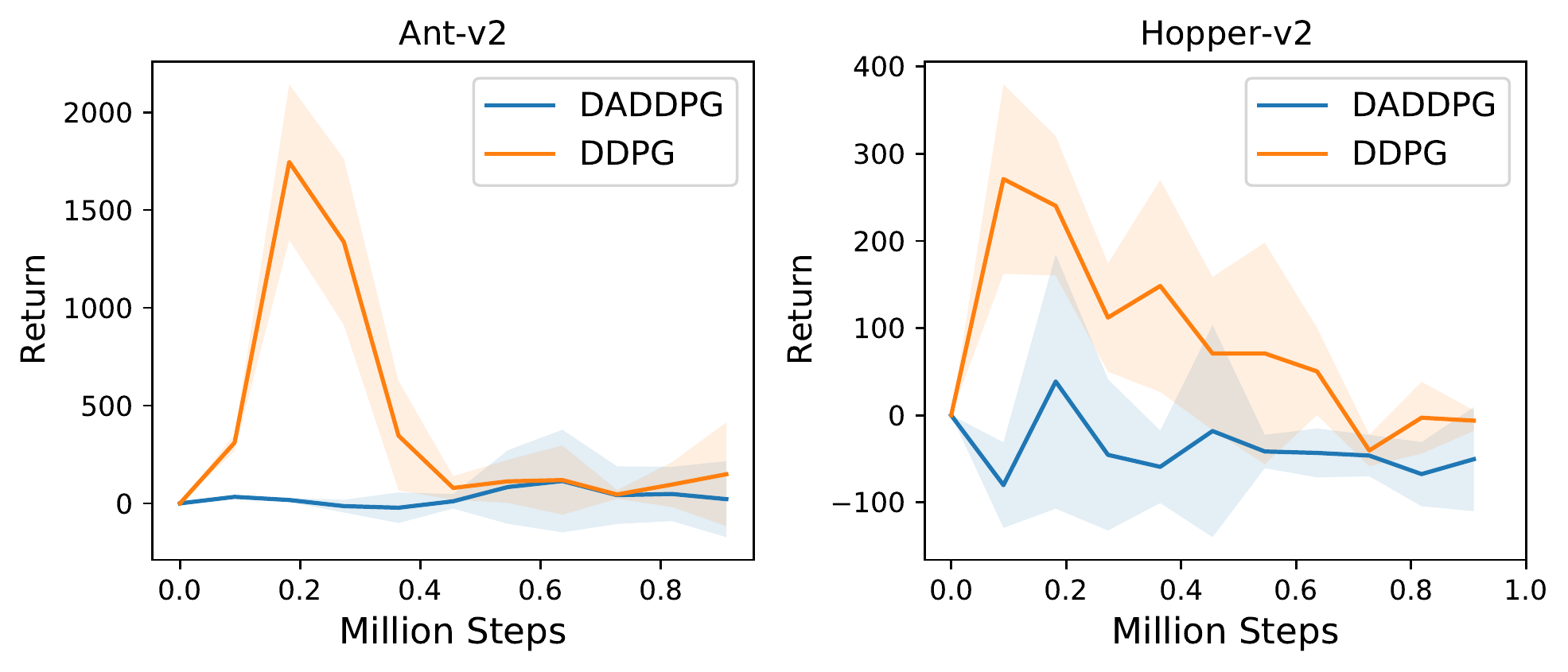}
    }\hspace{-2mm}
    \subfigure[DATD3 vs TD3]{
    \label{fig:datd3broader}
    \includegraphics[scale=0.32]{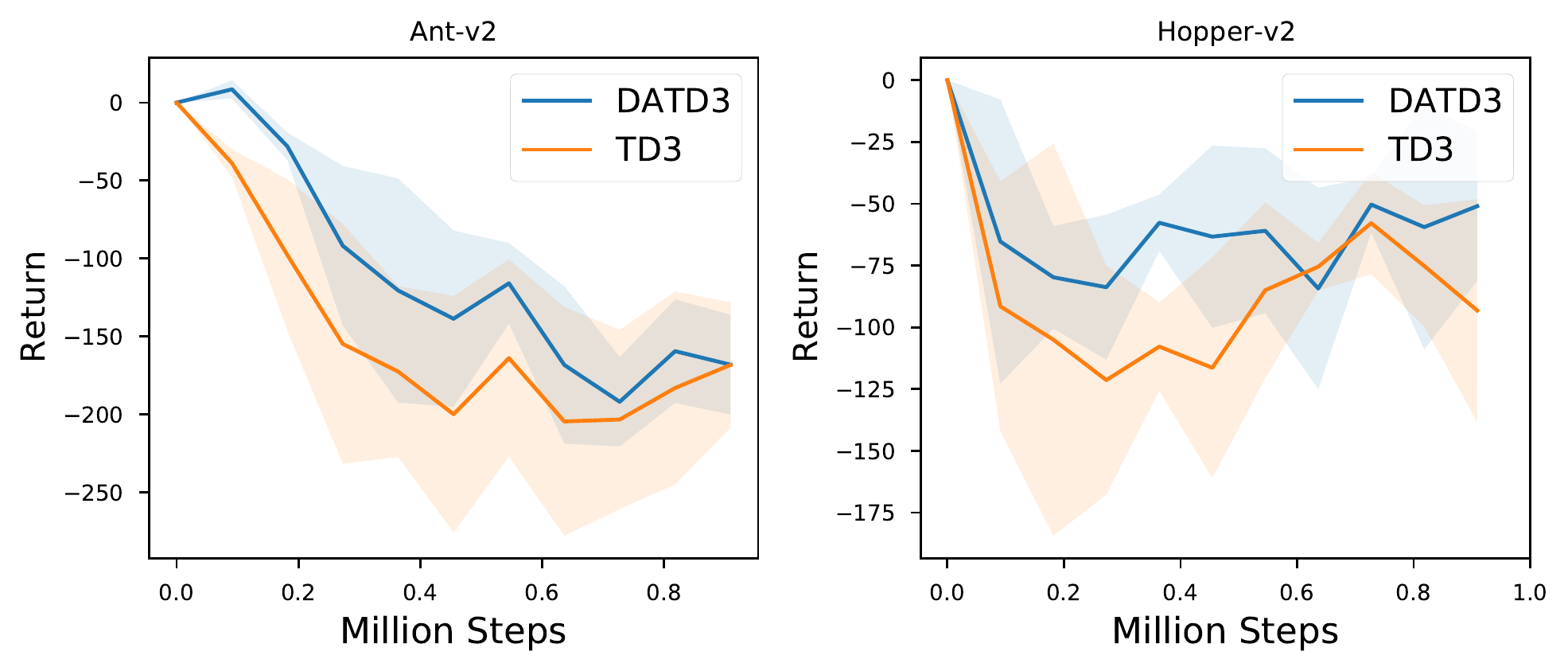}
    }\hspace{-2mm}
    \caption{Estimation bias comparison on Ant-v2 and Hopper-v2.}
    \label{fig:biascomp}
\end{figure}

\subsection{Bias Alleviation with DARC}
\label{sec:biasdarc}
In this section, we compare the estimation bias of DARC with DATD3 and TD3 on three MuJoCo environments, Ant-v2, Hopper-v2 and Walker2d-v2 and the result is shown in Fig \ref{fig:dracbias} where one could find that the bias of DARC is more conservative than that of DATD3 due to the mechanism of critic regularization. It is worth noting that it would be hard to give explicit relationship between the bias of DATD3 and DARC. One cannot simply conclude that the bias of DARC is smaller than DATD3 based on $\hat{V}(s;\nu)\le \hat{V}(s;\nu=0)$ because the critics in DARC are no longer independent but correlated with each other. DARC may preserves larger bias than that of TD3 (e.g., on Ant-v2 environment, see Fig \ref{fig:dracbias} for more details). DARC estimates the value function in a much softer way than that of DATD3 which is also more flexible on different tasks. Note that always taking maximum may cause slight overestimation bias which would do harm to the performance of the agent. DARC is also more stable and robust with critic regularization.

\begin{figure}
    \centering
    \includegraphics[scale=0.4]{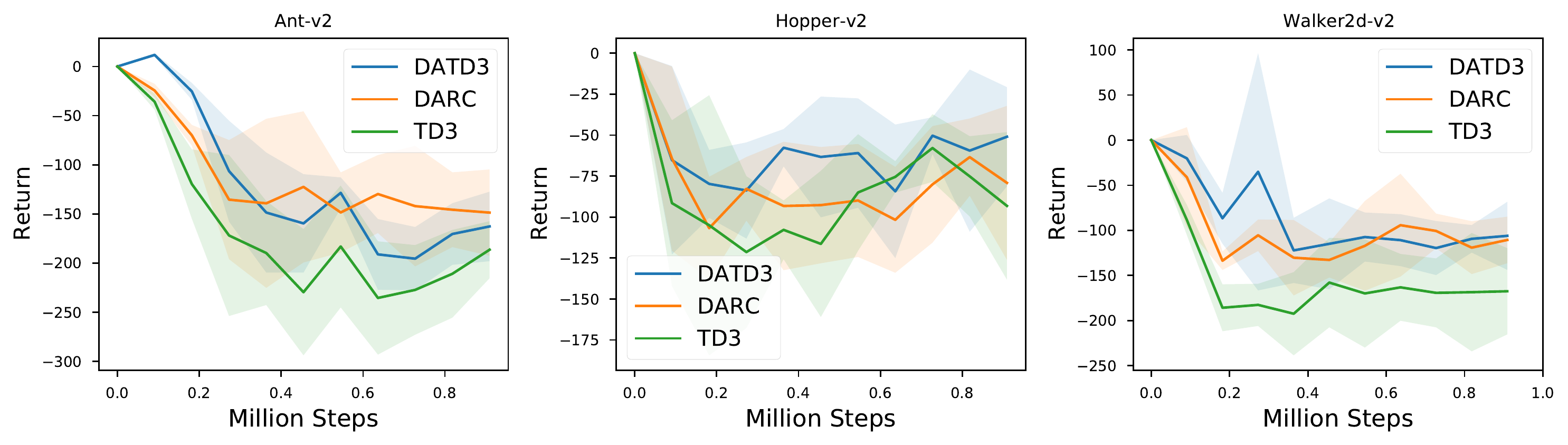}
    \caption{Bias comparison of DARC with DATD3 and TD3 on three MuJoCo environments.}
    \label{fig:dracbias}
\end{figure}

Moreover, we show in Table \ref{tab:criticreg} the mean critic deviance  $\mathbb{E}[Q_{\theta_1}(s,a) - Q_{\theta_2}(s,a)]$ in TD3 and DARC over 5 independent runs on three typical MuJoCo environments: Ant-v2, Hopper-v2 and Walker2d-v2 (1M steps). Denote the mean critic deviance from TD3 and DARC as $e_t,e_d$ respectively and then the reduction in Table \ref{tab:criticreg} is calculated via: $\mathrm{reduction} = \dfrac{e_t-e_d}{e_t}$. The results show that critic regularization significantly reduces the critic deviance, i.e., the value estimates from double critics are close to each other in DARC, leading to smaller uncertainty in value estimation from double critics.

\begin{table*}
\caption{Critic deviance reduction with critic regularization.}
\label{tab:criticreg}
\centering
\begin{tabular}{c|c|c|c}
\toprule
\textbf{Environment}  & \textbf{TD3} & \textbf{DARC} & \textbf{Reduction} \\
\midrule
Ant  & 1.6504 & \textbf{0.6062} & \textbf{63.27\%} \\
Hopper & 0.8844 & \textbf{0.0214} & \textbf{97.58\%} \\
Walker2d & 0.5711 & \textbf{0.0842} & \textbf{85.26\%} \\
\bottomrule
\end{tabular}
\end{table*}

\section{Proofs in Section 4}
\label{sec:theorem3}
\begin{theorem}
\label{appendixtheo:valueiteration}
Assume that the value error is bounded at the t-th iteration, i.e., $\|\max_{a\in\mathcal{A}} Q_t(s,a) - \hat{V}_t(s;\nu=0)\|_\infty \le \epsilon_t$ where $\epsilon_t$ is a time-variant upper bound. Assume that the policy execution error is bounded, i.e., $\|Q_{\theta_i}(s,\pi_{\phi_1}(s)) - Q_{\theta_i}(s,\pi_{\phi_2}(s))\|_\infty\le\epsilon_\pi, i=1,2$, and the critic deviance error is also bounded, i.e., $\forall s,a, \|Q_{\theta_1}(s,a) - Q_{\theta_2}(s,a)\|_\infty\le \epsilon_d$. Then for any iteration $t$, the difference between the optimal value function $V^*(s)$ and the value function $V(s)$ induced by the double actors satisfies:
\begin{equation*}
    \| V_t(s) - V^*(s)\|_\infty \le \gamma^t\|V_0(s) - V^*(s)\|_\infty + \sum_{k=0}^t \gamma^k \epsilon_k +  \dfrac{\nu}{1-\gamma}(2\epsilon_d+\epsilon_\pi).
\end{equation*}
\end{theorem}

\begin{proof}
Denote $a_1=\pi_{\phi_1}(s), a_2 = \pi_{\phi_2}(s)$. The target value is updated by a convex combination of value estimations from different policies, i.e.,
\begin{equation*}
    \hat{V}(s;\nu) = \nu \min (Q_1(s,a_1), Q_2(s,a_2))+(1-\nu)\max (Q_1(s,a_1), Q_2(s,a_2))
\end{equation*}

where $Q_1(s,a_1) = \min_{i=1,2}\left\{Q_{\theta_i}(s,a_1)\right\}$, is the minimal value of double critics under policy $\pi_{\phi_1}(s)$ and $Q_2(s,a_2) = \min_{j=1,2}\left\{Q_{\theta_j}(s,a_2)\right\}$ is the minimal value estimation of double critics under policy $\pi_{\phi_2}(s)$.

Then we have the following lemma:

\begin{lemma}
\label{lemma:q1q2}
For any state $s$, the distance between $Q_1(s,a_1)$ and $Q_2(s,a_2)$ is bounded by $2\epsilon_d+\epsilon_\pi$, i.e. $\| Q_2(s,a_2) - Q_1(s,a_1) \|_\infty \le 2\epsilon_d+\epsilon_\pi$, where $a_1=\pi_{\phi_1}(s), a_2 = \pi_{\phi_2}(s)$.
\end{lemma}

\begin{proof}
By the definition of $Q_1(s,a_1), Q_2(s,a_2)$, we have
\begin{equation*}
    \begin{aligned}
    &\quad \; \| Q_2(s,a_2) - Q_1(s,a_1) \|_\infty \\
    &\le \| Q_2(s,a_2) - Q_{\theta_2}(s,a_2) \|_\infty + \|  Q_{\theta_2}(s,a_2) - Q_1(s,a_1)\|_\infty \\
    &\le \| Q_2(s,a_2) - Q_{\theta_2}(s,a_2) \|_\infty + \| Q_{\theta_2}(s,a_2) - Q_{\theta_2}(s,a_1) \|_\infty + \| Q_{\theta_2}(s,a_1) - Q_1(s,a_1) \|_\infty \\
    &\le 2\epsilon_d + \epsilon_\pi.
    \end{aligned}
\end{equation*}

Note that $\| Q_2(s,a_2) - Q_{\theta_2}(s,a_2) \|_\infty \le \epsilon_d$ because $Q_2(s,a_2)$ can only be chosen between $Q_{\theta_1}(s,a_2)$ and $Q_{\theta_2}(s,a_2)$, and so does the error bound between $Q_{\theta_2}(s,a_1)$ and $Q_1(s,a_1)$.
\end{proof}

Notice that the max operator is non-expansive, then at the $(t+1)$-th iteration, we have
\begin{equation*}
    \begin{aligned}
    &\quad \; |V_{t+1}(s) - V^*(s)| \\
    &=|\hat{V}(s;\nu) - \max_{a\in\mathcal{A}}Q^*(s,a)| \\
    &\le |\hat{V}(s;\nu) - \max_{a\in\mathcal{A}}Q_{t+1}(s,a)| + |\max_{a\in\mathcal{A}}Q_{t+1}(s,a) - \max_{a\in\mathcal{A}}Q^*(s,a)| \\
    &= | \nu \min \{Q_1(s,a_1), Q_2(s,a_2)\}+(1-\nu)\max \{Q_1(s,a_1), Q_2(s,a_2)\} - \max_{a\in\mathcal{A}}Q_{t+1}(s,a)| \\
    &\qquad + |\max_{a\in\mathcal{A}}Q_{t+1}(s,a) - \max_{a\in\mathcal{A}}Q^*(s,a)| \\
    &\le | \max \{Q_1(s,a_1), Q_2(s,a_2)\} - \max_{a\in\mathcal{A}}Q_{t+1}(s,a) | + \nu | \min \{Q_1(s,a_1), Q_2(s,a_2)\} \\
    & \qquad - \max \{Q_1(s,a_1), Q_2(s,a_2)\} |  + \max_{a\in\mathcal{A}} |Q_{t+1}(s,a) - Q^*(s,a)| \\
    &\overset{(i)}{\le} \epsilon_t + \nu | \min \{Q_1(s,a_1), Q_2(s,a_2)\} - \max \{Q_1(s,a_1), Q_2(s,a_2)\} | + \gamma \max_{s\in\mathcal{S}} |V_t(s) - V^*(s)| \\
    \end{aligned}
\end{equation*}

where (i) holds due to the fact that
\begin{equation*}
    | Q_{t+1}(s,a) - Q^*(s,a) | \le \gamma \max_{s^\prime\in\mathcal{S}}| V_t(s^\prime) - V^*(s^\prime) |.
\end{equation*}

Consider $\infty$-norm and we have

\begin{equation*}
    \begin{aligned}
    & \| V_{t+1}(s)-V^*(s) \|_\infty \le \epsilon_t + \nu \| \min \{Q_1(s,a_1), Q_2(s,a_2)\} - \max \{Q_1(s,a_1), Q_2(s,a_2)\} \|_\infty \\
    &\qquad \qquad \qquad \qquad \qquad \qquad + \gamma \|V_t(s) - V^*(s)\|_\infty
    \end{aligned}
\end{equation*}

As we only take min and max between two numbers, then

\begin{equation*}
    \begin{aligned}
    \| \min \{Q_1(s,a_1), Q_2(s,a_2)\} - \max \{Q_1(s,a_1), Q_2(s,a_2)\} \|_\infty = \| Q_2(s,a_2) - Q_1(s,a_1) \|_\infty.
    \end{aligned}
\end{equation*}

Then by using Lemma \ref{lemma:q1q2}, we have
$$
\| \min \{Q_1(s,a_1), Q_2(s,a_2)\} - \max \{Q_1(s,a_1), Q_2(s,a_2)\} \|_\infty \le 2\epsilon_d+\epsilon_\pi .
$$

Taken together, we get the relationship between $(t+1)$-th iteration and $t$-th iteration:
\begin{equation*}
    \|V_{t+1}(s) - V^*(s)\|_\infty \le \epsilon_t + \nu (2\epsilon_d + \epsilon_\pi) + \gamma \|V_t(s) - V^*(s)\|_\infty.
\end{equation*}

Doing iteration and we could reach to the desired conclusion:
\begin{equation*}
    \begin{aligned}
    \| V_t(s) - V^*(s) \|_\infty &\le \gamma^t\|V_0(s) - V^*(s)\|_\infty + \sum_{k=1}^t \gamma^k (\epsilon_k + \nu(2\epsilon_d + \epsilon_\pi)) \\
    &= \gamma^t\|V_0(s) - V^*(s)\|_\infty + \sum_{k=1}^t \gamma^k\epsilon_k + \dfrac{\nu(2\epsilon_d+\epsilon_\pi)}{1-\gamma}.
    \end{aligned}
\end{equation*}
\end{proof}

\section{Hyperparameter Setup}
\label{sec:hyperparameters}
The network configuration for DDPG, TD3, and DARC are similar, where we use fine-tuned DDPG as is suggested in TD3 instead of the vanilla DDPG. The detailed hyperparameters for baseline algorithms and DARC are shown in Table \ref{tab:hyperparameter} where we adopt different hyperparameters for Humanoid-v2 environment as all algorithms fail in this task if we use the default hyperparameters setup as other environments. We instead use the parameters suggested in TD3 \cite{TD3} such that algorithms could learn well in this environment. For simplicity, we set the regularization parameter $\lambda$ as $0.005$. The weighting coefficient $\nu$ differs in different tasks: $0.22$ for Ant-v2, $0.1$ for HalfCheetah-v2, $0.15$ for Hopper-v2, $0.05$ for Humanoid-v2, $0.12$ for Walker2d, $0.2$ for AntPybullet (AntMuJoCoEnv-v0), $0.15$ for Walker2dPybullet (Walker2DMuJoCoEnv-v0), and a relatively large $\nu = 0.4$ for BipedalWalker-v3. The random seeds we adopt are $1$-$5$ for simplicity.

\begin{table*}
\caption{Hyperparameters setup for baseline algorithms and DARC}
\label{tab:hyperparameter}
\centering
\begin{tabular}{lrr}
\toprule
\textbf{Hyperparameter}  & \textbf{Humanoid-v2} & \textbf{Other environments} \\
\midrule
Shared & \\
\qquad Actor network  & $(256,256)$ & $(400,300)$ \\
\qquad Critic network & $(256,256)$ & $(400,300)$ \\
\qquad Batch size & $256$ & $100$ \\
\qquad Learning rate & $3\times 10^{-4}$ & $10^{-3}$ \\
\qquad Optimizer & \multicolumn{2}{c}{Adam} \\
\qquad Discount factor & \multicolumn{2}{c}{$0.99$} \\
\qquad Replay buffer size & \multicolumn{2}{c}{$10^6$}  \\
\qquad Warmup steps & \multicolumn{2}{c}{$10^4$} \\
\qquad Exploration noise & \multicolumn{2}{c}{$\mathcal{N}(0,0.1)$} \\
\qquad Noise clip & \multicolumn{2}{c}{$0.5$} \\
\qquad Target update rate & \multicolumn{2}{c}{$5\times 10^{-3}$} \\

\midrule
TD3  & \\
\qquad Target update interval & \multicolumn{2}{c}{$2$} \\
\qquad Target noise & \multicolumn{2}{c}{$0.2$} \\
\midrule
SAC & \\
\qquad Reward scale & \multicolumn{2}{c}{$1$} \\
\qquad Entropy weight & \multicolumn{2}{c}{$0.2$} \\
\qquad Maximum log std & \multicolumn{2}{c}{$2$} \\
\qquad Minimum log std & \multicolumn{2}{c}{$-20$} \\
\midrule
DARC & \\
\qquad Regularization parameter $\lambda$ & \multicolumn{2}{c}{$0.005$} \\
\qquad Target noise $\bar{\sigma}$ & \multicolumn{2}{c}{$0.2$} \\
\bottomrule
\end{tabular}
\end{table*}

\section{Extensive Structure Comparison of Different Algorithms}

In this section, we compare the structure and specify whether value estimation correction is conducted and regularization strategy is applied among different baseline algorithms, recent methods and our proposed methods, i.e., we provide a checklist on the comparison of different components of different algorithms. The details are available in Table \ref{tab:structure}.

Note that we present a structure table in Table \ref{tab:structurecomparison} of the main text, where the improvement of each algorithm compared with DDPG is calculated by averaging the relative improvement on mean final scores over 5 independent runs with random seed $1$-$5$ on each MuJoCo \cite{todorov2012mujoco} environment, where two decimal places are reserved. To be specific, assume there are $k$ environments and denote the mean final score of DDPG and an algorithm $A$ as $d_k, a_k$ respectively, then the improvement is calculated via:
\begin{equation*}
    \mathrm{improvement} = \dfrac{1}{k}\sum_{i=1}^k \dfrac{a_i-d_i}{d_i}
\end{equation*}

\begin{table*}
\caption{Extensive algorithmic structure and component comparison}
\label{tab:structure}
\centering
\small
\begin{tabular}{c|c|c|c|c}
\toprule
\textbf{Algorithms}  & \textbf{Double Actors} & \textbf{Double Critics}  & \textbf{Value Correction} & \textbf{Regularization} \\
\midrule
DDPG \cite{lillicrap2015continuous}  & \XSolidBrush & \XSolidBrush & \XSolidBrush & \XSolidBrush \\
TD3 \cite{fujimoto2018addressing}  & \XSolidBrush & \CheckmarkBold & \CheckmarkBold & \XSolidBrush \\
SAC \cite{haarnoja2018soft}  & \XSolidBrush & \CheckmarkBold & \CheckmarkBold & \CheckmarkBold \\
TADD \cite{wu2020reducing}  & \XSolidBrush & \XSolidBrush   & \CheckmarkBold & \XSolidBrush \\
SD3 \cite{pan2020softmax}  & \CheckmarkBold & \CheckmarkBold & \CheckmarkBold & \XSolidBrush \\
DADDPG (this work) & \CheckmarkBold & \XSolidBrush  & \CheckmarkBold & \XSolidBrush \\
DATD3 (this work) & \CheckmarkBold & \CheckmarkBold & \CheckmarkBold & \XSolidBrush \\
DARC (this work)  & \CheckmarkBold & \CheckmarkBold & \CheckmarkBold & \CheckmarkBold \\
\bottomrule
\end{tabular}
\end{table*}

\section{Computing Infrastructure}
In this section, we provide a detailed description of the computing infrastructure that we use to run all of the baseline algorithms and experiments in Table \ref{tab:computing}.

\begin{table*}
\caption{Computing Infrastructure}
\label{tab:computing}
\centering
\begin{tabular}{c|c|c}
\toprule
\textbf{CPU}  & \textbf{GPU} & \textbf{Memory} \\
\midrule
AMD EPYC 7452  & RTX3090$\times$8 & 288GB \\
\bottomrule
\end{tabular}
\end{table*}

\end{document}